\documentclass{article}

\usepackage[preprint]{neurips_2026}

\usepackage[utf8]{inputenc} 
\usepackage[T1]{fontenc}    
\usepackage[colorlinks, linkcolor=red, urlcolor=cyan, citecolor=blue]{hyperref}
\usepackage{url}            
\usepackage{booktabs}       
\usepackage{amsfonts}       
\usepackage{nicefrac}       
\usepackage{microtype}      
\usepackage{xcolor}         

\usepackage{wrapfig}
\usepackage{dsfont}
\usepackage[noend]{algorithmic}
\usepackage{algorithm}

\usepackage{multirow}
\usepackage{amsmath,amsfonts,bm}
\usepackage{amssymb}
\usepackage{mathtools}
\usepackage{amsthm}
\usepackage[capitalize,noabbrev]{cleveref}

\usepackage{pifont}
\usepackage[normalem]{ulem}

\theoremstyle{plain}
\newtheorem{theorem}{Theorem}[section]

\newtheorem{lemma}[theorem]{Lemma}
\newtheorem{corollary}[theorem]{Corollary}
\theoremstyle{definition}
\newtheorem{definition}[theorem]{Definition}

\theoremstyle{remark}

\newcommand{\smodel}{TreeX}

\newcommand{\R}{\mathbb{R}}

\def\rmA{{\mathbf{A}}}

\def\rmX{{\mathbf{X}}}

\def\va{{\bm{a}}}

\def\vh{{\bm{h}}}

\def\vm{{\bm{m}}}

\def\vo{{\bm{o}}}
\def\vp{{\bm{p}}}

\def\vz{{\bm{z}}}

\def\gC{{\mathcal{C}}}
\def\gD{{\mathcal{D}}}
\def\gE{{\mathcal{E}}}

\def\gL{{\mathcal{L}}}

\def\gN{{\mathcal{N}}}

\def\gX{{\mathcal{X}}}

\usepackage{xcolor}
\usepackage{tcolorbox}
\tcbuselibrary{listings,breakable}

\newtcblisting{rulelogbox}{
  colback=gray!4,
  colframe=gray!45,
  listing only,
  listing options={
    basicstyle=\ttfamily\scriptsize,
    breaklines=true,
    columns=fullflexible,
    keepspaces=true,
    showstringspaces=false
  },
  left=1mm,
  right=1mm,
  top=1mm,
  bottom=1mm,
  arc=1mm,
  boxrule=0.4pt
}

\usepackage{tikz}
\newcommand{\bcircled}[1]{%
  \tikz[baseline=-0.65ex]{
    \node[
      shape=circle,
      fill=black,
      text=white,
      inner sep=0.6pt,
      minimum size=0.7em,
      font=\scriptsize\bfseries
    ] (char) {#1};
  }%
}

\newcommand{\mymethod}{TreeX}

\title{Model-Level GNN Explanations via Rule-to-Graph Readout for Logit Reconstruction}

\author{
  Shengyao Lu\\{University of Victoria} \\
  \texttt{shengyaolu@uvic.ca} \\
  \And 
  Jiuding Yang\\{University of Alberta}\\
  \texttt{jiuding@ualberta.ca}\\
  \And 
  Aedan J. DeFrates\\LSU ATHENA Lab\\
  \texttt{adefra5@lsu.edu}
  \And
  Keith G. Mills\\LSU ATHENA Lab \\
  \texttt{keith.mills@lsu.edu}
  \And 
  Baochun Li\\ University of Toronto \\
  \texttt{bli@ece.toronto.edu}
  \And 
  Di Niu\\{University of Alberta}\\
  \texttt{dniu@ualberta.ca}\\
}

\begin{document}

\maketitle

\begin{abstract}
We propose {\mymethod}, a novel model-level GNN explanation framework that shifts the explanation target from class-wise rule extraction to rule-based logit reconstruction. 
{\mymethod} recasts the graph-level readout of a pretrained GNN as a weighted rule-level readout: grounded subgraph concepts are composed into logical rules, rule embeddings are computed directly from their symbolic structure, and active rules are passed through the frozen classifier head to reconstruct the GNN's raw multiclass logits. 
As a result, {\mymethod} provides global explanations that remain instantiable on unseen graphs, support subgraph-level grounding, and admit rule-level contribution analysis at test-time.
Experiments on three synthetic and two real-world graph classification benchmarks show that {\mymethod} faithfully reconstructs the base GNN's raw multiclass logits, achieving high probability-level fidelity across datasets. 
Rule-level ablations further demonstrate that the identified critical rules actively support the predicted class while suppressing non-target classes, suggesting that they act as functional units rather than merely serving as post-hoc symbolic artifacts.
Compared with prior class-wise rule-based explainers, {\mymethod} achieves competitive or better prediction agreement while being up to \(20\times\) faster, and additionally provides rule weights, test-time grounding, and logit-level contribution analysis.\footnote{
Earlier unpublished online drafts STExplainer~\cite{lu2024stexplainer} and TreeX~\cite{lu2025treex} are previous versions of this work. 
}
\end{abstract}



\begin{figure*}[ht]
  \centering
  \includegraphics[width=0.88\textwidth]{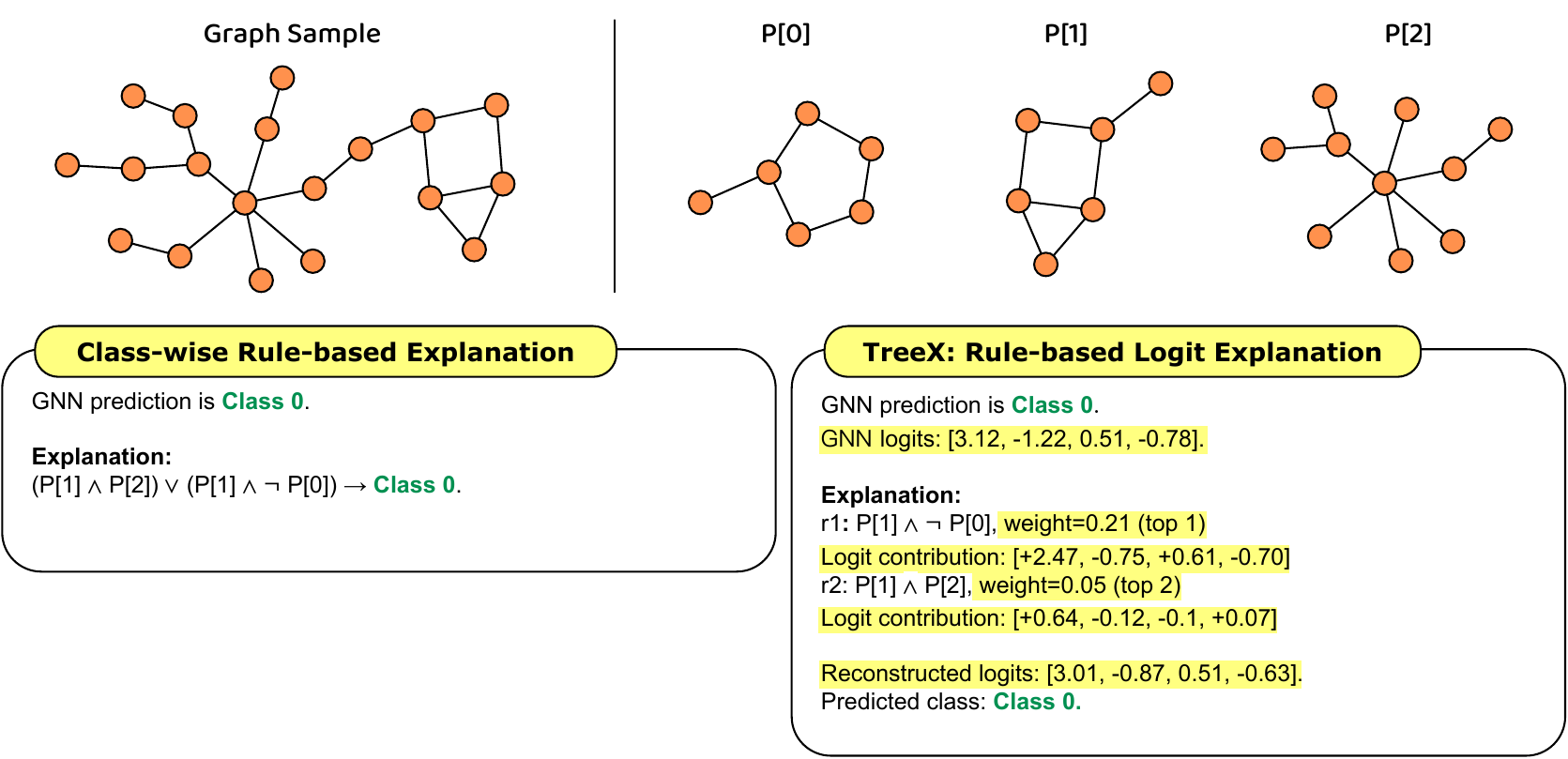}
  \caption{Prior rule-based explainers identify symbolic patterns associated with a predicted class, but do not explain how these rules produce the model’s raw prediction logits. In {\mymethod}, each grounded symbolic rule contributes a numeric vector to the multiclass logits; these rules are combined to faithfully reconstruct the pretrained GNN’s original logits, enabling rule-level analysis of how rules support the predicted class or suppress competing classes.
  }
  \vspace{-0.05in}
  \label{fig:ad}
\end{figure*}

\section{Introduction} \label{sec:intro}
Despite the strong predictive performance of graph neural networks (GNNs)~\cite{hamilton2017inductive, xupowerful, schlichtkrull2018modeling}, understanding their decision logic remains challenging. 
Early studies mainly focus on instance-level explanations~\cite{yuan2021explainability, ying2019gnnexplainer, luo2020parameterized}, identifying nodes or subgraphs responsible for a single prediction. 
However, local explanations explain predictions, not the predictor: they do not reveal whether the same evidence is used systematically across the dataset, whether the model relies on spurious shortcuts, or how multiple concepts jointly define class-level decision boundaries. 
Model-level explainability is therefore essential for auditing, debugging, and validating the behavior of a trained GNN.

Most existing model-level GNN explainers~\cite{azzolin2022global, xuanyuan2023global, armgaan2024graphtrail, geng2026logicxgnn, yuan2020xgnn, wang2022gnninterpreter, yu2025mage} ask what graph pattern or symbolic rule is associated with a class. 
We ask a different question: can the \textit{raw prediction logits} of a trained GNN be reconstructed from human-interpretable rules?
As illustrated in~\cref{fig:ad}, this changes the role of rules from post-hoc class descriptors to functional computational units.
Traditional class-wise rules indicate which symbolic conditions correlate with a class prediction, but they do not explain how multiple rules jointly produce the GNN’s final decision.
In our work, we instead target rule-based reconstruction of the pretrained GNN’s multiclass logits, enabling explanations of how rules jointly produce the final decision by supporting the predicted class, suppressing competing classes, and shaping the decision margin among alternatives.


\newcommand{\contrib}[2]{%
  \bcircled{#1}~\uline{\textbf{\textit{#2.}}}%
}

In this paper, we propose {\mymethod}, a model-level GNN explainability framework that shifts the explanation target from post-hoc class-wise summaries to rule-based reconstruction of the base GNN's raw prediction logits.
Instead of learning an unconstrained surrogate, {\mymethod} intervenes at the readout interface of a pretrained GNN: the message-passing layers and classifier head are kept frozen, while the original node-to-graph readout is replaced by a weighted rule-to-graph readout.
Our contributions are as follows.
\bcircled{1} \textbf{A new explanation target.}
We formulate model-level GNN explanation as rule-based logit reconstruction, moving beyond explanations that only associate graph patterns or class-wise rules with predicted labels~\cite{yu2025mage,geng2026logicxgnn,xuanyuan2023global}.
\bcircled{2} \textbf{A symbolic readout interface.}
We introduce a rule-to-graph readout that treats grounded logical rules as computational units: active rule embeddings are weighted, aggregated into a graph-level representation, and passed through the frozen classifier head to reconstruct the GNN's multiclass logits.
\bcircled{3} \textbf{Grounded and compositional rule construction.}
{\mymethod} mines subtree-inspired subgraph concepts that provide both continuous embeddings for rule construction and concrete subgraph evidence for test-time grounding. It further learns OR-of-CNF rules that preserve alternative realizations within required constraints, yielding a more compositional and structurally expressive rule space than DNF-style formulations under the same rule budget.
\bcircled{4} \textbf{Functional and Faithful Rule Explanations.}
Experiments show that {\mymethod} faithfully reconstructs GNN logits, and that the top active rules selected for each instance increase the GNN-predicted class logit while suppressing non-target logits.
Compared with prior class-wise rule explainers, {\mymethod} remains competitive in prediction agreement while being substantially faster, and additionally supports rule weights, test-time grounding, and rule-level ablation.

\section{Related Work} \label{sec:related_work}
Early studies on GNN explainability focused primarily on instance-level explanations~\cite{bui2024explaining, lu2024eigsearch, feng2022degree, yuan2021explainability, zhang2022gstarx, lu2024goat, ying2019gnnexplainer, luo2020parameterized, shan2021reinforcement, lin2021generative}, which identify a small subgraph or feature subset supporting a single prediction. While effective for local inspection, they cannot summarize how relational evidence recurs across many instances, which is essential for task understanding and GNN model debugging. This limitation has motivated a growing line of work on \textit{model-level} or \textit{global} explanations for GNNs.

Generative approaches such as XGNN~\cite{yuan2020xgnn}, GNNInterpreter~\cite{wang2022gnninterpreter}, PAGE~\cite{shin2024page}, and Gen-GraphEx~\cite{2025gengraphex} seek model-level explanations by generating class-discriminative graphs, prototypes, or distributions of explanation graphs, while MAGE~\cite{yu2025mage} further constrains this process with class-specific motif building blocks to improve structural validity. However, they share a central limitation: a high-scoring generated pattern is not necessarily a faithful or data-grounded rationale. 

Rule or concept-based approaches such as 
GCNeuron~\cite{xuanyuan2023global}, GLGExplainer~\cite{azzolin2022global}, GraphTrail~\cite{armgaan2024graphtrail} and LogicXGNN~\cite{geng2026logicxgnn} move beyond prototype generation toward more explicit symbolic explanations. 
GCNeuron requires predefined rule templates and explains the model mainly through a neuron-aligned-concept summary.
GLGExplainer builds its global Boolean formulas on top of auxiliary local explanations, making its final logic highly dependent on the quality of the local explainer. 
GraphTrail utilizes Shapley heuristics for concept mining, making it inefficient. 
LogicXGNN is the closest related work because it also emphasizes grounded rule-based explanations. 
However, LogicXGNN extracts predicates and class-wise DNF rules through discrete decision-tree procedures. 
This design is effective for obtaining grounded class rules, but it does not naturally provide continuous rule embeddings that can be composed into a readout representation and passed through the frozen classifier head. 
Therefore, we do not use LogicXGNN-style predicate extraction as the rule-construction backbone of {\mymethod}. 
Our goal is not only to obtain symbolic class-wise rules, but to build rule embeddings that reconstruct the base GNN's multiclass logits.


\section{Preliminaries} \label{sec:prelim}


\paragraph{Graph Neural Networks (GNNs).}
\label{sec:gnn}
Let \(G=(V,E)\) be a graph with node set \(V\), edge set \(E\), and \(N=|V|\) nodes. 
A GNN~\cite{kipf2017semisupervised,xupowerful,hamilton2017inductive} model $f(\rmX,\rmA)$ maps the node features $\rmX \in \R^{N\times d}$ of dimension $d$ and the adjacency matrix $\rmA \in \R^{N\times N}$ indicating the existence or absence of edges $E$ to a target output such as a node label, graph label, or edge label. 
In this paper, we study graph classification and assume a pretrained classifier $f(\rmX,\rmA)$ that is frozen during explanation learning.
Let $l$ denote a message-passing layer. At layer $l$, the GNN aggregates neighbourhood information for each node $v\in V$ from its previous-layer representation $\vh^{(l-1)}_v$ and produces the updated representation $\vh^{(l)}_v$. As shown in \cref{fig:overview}a), after $L$ layers, the graph-level representation is obtained by a readout function
$\vh_G = \operatorname{READOUT}\big(\{\vh_v^{(L)} : v \in V\}\big),$ followed by a classifier head $\phi(\cdot)$ that maps $\vh_G$ to raw prediction logits $\vo_G$. 



\paragraph{Conjunctive Normal Form (CNF).}
The symbolic layer of our framework is built from propositional logic over learned subgraph concepts. A \textit{literal} is either a positive atom $P[k]$ or its negation $\neg P[k]$, where $P[k]$ means that concept $k$ is present in the graph. A \textit{clause} is a disjunction of literals,
$\ell_1 \vee \ell_2 \vee \cdots \vee \ell_m,$
and a CNF rule is a conjunction of clauses,
$r = c_1 \wedge c_2 \wedge \cdots \wedge c_M.$


We adopt CNF because we aim to explain a target class through a small set of jointly required concept-level constraints, while still allowing each constraint to have multiple alternative realizations. This directly matches the AND-of-OR structure of CNF and provides a compact symbolic interface for the rule learner introduced next.

\section{Methodology} \label{sec:method}
\begin{figure*}[t]
  \centering
  \vspace{-0.45in}
  \includegraphics[width=\textwidth]{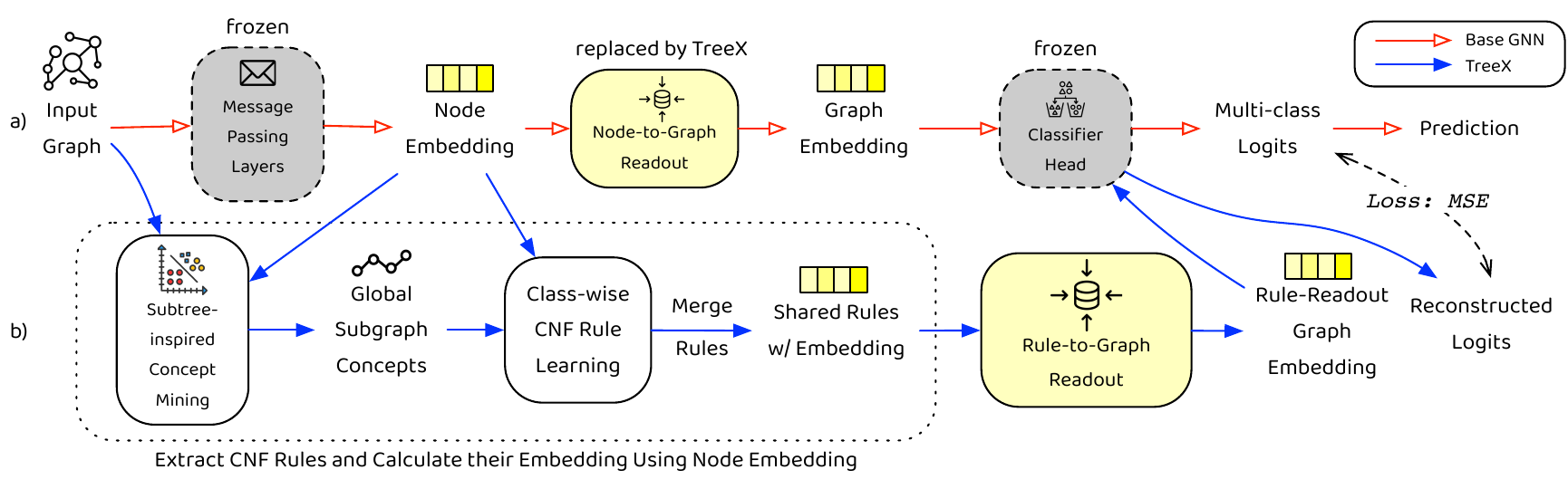}
  \vspace{-0.2in}
  \caption{(a) Pipeline of a typical GNN. (b) Overview of {\mymethod}.
{\mymethod} recasts the readout layer of the base GNN as a rule-level readout.
From the node embeddings, it mines global subgraph concepts, learns class-wise CNF candidate rules, merges them into a shared rule bank, and constructs a rule-readout graph embedding.
Eventually, it produces a multiclass logit vector, trained by minimizing the MSE between the reconstructed and base-GNN logits.}
  \vspace{-0.15in}
  \label{fig:overview}
\end{figure*}

The full pipeline of {\mymethod} is illustrated in~\cref{fig:overview}: we \textbf{first} build a global concept space, \textbf{then} use class-wise CNF rule discovery only to generate candidate symbolic rules, and \textbf{finally} merge them into one shared weighted rule bank that mimics the base GNN's multiclass behavior.

\subsection{Subtree-Inspired Concept Mining}
\paragraph{Subtree-inspired local subgraph concept extraction.}
\label{sec:local_concept}

We use a subtree-inspired local concept mining mechanism in order to ground the global rules easier on local instances later at test time. We theoretically justify this design in \underline{\bf Appendix~\ref{app:just_subtree}}. 

Specifically, at training time, for each graph $G$, we cluster the $l$-th layer node embeddings within the graph into at most $K_{\text{local}}$ clusters using KMeans~\cite{macqueen1967some}. Let $C_m \subseteq V$ denote one such hidden-state cluster. Each cluster is converted into a support subgraph by aggregating the $l$-hop receptive field of all its member nodes, i.e., aggregating the $l$-hop rooted subtrees.
If $\gE_l(v)$ denotes the set of edges inside the $l$-hop neighborhood of node $v$, we define the support count of edge $e$ under cluster $C_m$ as:
\begin{equation}
    c_m(e) = \sum_{v \in C_m} \mathbf{1}[e \in \gE_l(v)].
\end{equation}
We retain only edges with sufficiently strong shared support, where we set $Z_m=5$ in our experiments:
\begin{equation}
    E_m^{*} = \{ e \in E \mid c_m(e) \ge \min(\max_{e'} c_m(e'), Z_m) \}.
\end{equation}
This step biases subgraph concepts toward structurally stable evidence rather than idiosyncratic noise from a single node. 
The local concept embedding is then calculated by aggregating the cluster-wise subtree embedding, which is typically a mean aggregator 
$
\vz_{m}=\operatorname{AGG}^{\text{local}} \left( \left\{ \vh_v^{(l)}: v \in C_m\right\}\right).
$

\paragraph{Mode-based global concept mining.}
\label{sec:global_concept}
All local concept embeddings extracted from the \textit{training split} are then clustered again with KMeans into $K_{\text{global}}$ global concept clusters. We take the cluster centroid as the global concept embedding. 

For each global cluster, our method selects one representative local subgraph by a two-stage lexicographic majority filter followed by a centroid-distance tie-break. Specifically, among all local concepts assigned to a global cluster, it first keeps only those matching the \textit{mode} of the shape signature 
\begin{equation*}
    s^{\text{shape}} = (|V_{\text{subgraph}}|,\; |E_{\text{subgraph}}|),
\end{equation*}
then from the remaining, keeps only those matching the \textit{mode} of the degree signature
\begin{equation*}
    s^{\text{deg}} = (\max \text{degree},\; \operatorname{mean} \text{degree}),
\end{equation*}
and finally selects, from the remaining candidates, the concept closest to the KMeans centroid. This representative subgraph is the stored global concept prototype. 

We use a two-stage filter because degree is only meant to refine candidates that already match the dominant shape. If we instead take the mode of the joint signature $(s^{\text{shape}}, s^{\text{deg}})$ in one step, shape-mismatched subgraphs can still influence the degree pattern and make the selection less robust.

After the global concept bank is fixed from the training split, every graph instance is projected into this shared concept space, yielding the concept mask:
\begin{equation*}
    M_{gk} =\text{number of local concept occurrences assigned to global concept } k \text{ in graph } g.
\end{equation*}

\subsection{Neuro-Symbolic Rule Learning over Global Subgraph Concepts}
\label{sec:neurosym}
Unlike arbitrary Boolean formulas \cite{azzolin2022global, armgaan2024graphtrail} that introduce a much less structured search space, and DNF (OR-of-ANDs) \cite{geng2026logicxgnn} that flattens the explanation into alternative conjunctions rather than explicitly modeling alternatives within each required constraint, we adopt an OR-AND-OR relaxation for symbolic rule learning. That is, OR over literals inside clause, AND over clauses inside rule, OR over rules for class decision. This design preserves the inner structure of alternative constraints and allows us to further learn weights over the outermost OR later, showing which active rules matter more.

\paragraph{Class-specific literalization.} 
The symbolic branch is trained class by class. For a target class $t$, it first selects the most class-discriminative global concepts on the training split. Define
\begin{equation}
    \Delta_{k}^{+}(t)= \mathbb{E}[M_{gk} \mid y_g=t]-\mathbb{E}[M_{gk} \mid y_g \neq t],
\end{equation}
\begin{equation}
    \Delta_{k}^{-}(t)= \mathbb{E}[M_{gk} \mid y_g \neq t] -\mathbb{E}[M_{gk} \mid y_g=t].
\end{equation}
The top concepts under $\Delta_k^{+}(t)$ are used to instantiate positive literals $P[k]$ indicating the existence of global concept $k$, while the top concepts under $\Delta_k^{-}(t)$ instantiate negative literals $\neg P[k]$ indicating the absence of $k$.

Each graph is then converted into a literal feature vector $x_g \in [0,1]^L$, where $L$ is number of literals. The positive concept literal is scored by a smooth presence, the negative literal is its complement. 
\begin{equation}
    \label{eq:literal_value}
    x_{g,P[k]}=1 - \exp\!\left(-\frac{M_{gk}}{\tau}\right),
    \qquad
    x_{g,\neg P[k]} = 1 - x_{g,P[k]}.
\end{equation}
where $\tau\in [0,1]$ is the presence temperature.

\paragraph{Differentiable CNF rule learning.}
Given the literal matrix $X \in [0,1]^{N \times L}$ for a target class $t$, we train a rule learner that fits a differentiable CNF-style rule system in a one-vs-rest manner. 
We use a hierarchical differentiable logic parameterization.

The model maintains trainable literal logits $\alpha_{r,c,\ell}$ and clause logits $\beta_{r,c}$, converted to gates via the sigmoid:
\begin{equation}
    a_{r,c,\ell} = \sigma(\alpha_{r,c,\ell}), \qquad
    b_{r,c} = \sigma(\beta_{r,c}).
\end{equation}
For a graph $g$, the soft satisfaction of clause $c$ is computed as a G{\"o}del t-conorm~\cite{zadeh1965fuzzy} over gated literal activations to ensure at least one literal is strongly present:
\begin{equation}
    s_{g,r,c}=\max_{\ell}\left(a_{r,c,\ell}\, x_{g,\ell}\right).
\end{equation}
The rule then applies a product-style soft AND~\cite{goguen1969logic} across clauses, in which deselected clauses act as multiplicative identities:
\begin{equation}
    q_{g,r}=\prod_c\left[1 - b_{r,c}\left(1 - s_{g,r,c}\right)\right].
\end{equation}
At the top level, class prediction is formed by a probabilistic-sum~\cite{pearl1988probabilistic}, equivalently a noisy-OR style disjunction, aggregating across candidate rules: 
\begin{equation}
    p_g^{(t)}=1 - \prod_r \left(1 - q_{g,r}\right).
\end{equation}
The learner is optimized with binary cross-entropy plus sparsity and gate-sharpening regularizers:
\begin{equation}
\label{eq:cnf_bce}
    \gL_{\text{CNF}}^{(t)}=\operatorname{BCE}(p_g^{(t)}, y_g^{(t)})\;+\;\lambda_{\text{lit}} \, \overline{a}\;+\;\lambda_{\text{clause}} \, \overline{b}\;+\;\gamma\left(\overline{a(1-a)} + \overline{b(1-b)}\right),
\end{equation}
where $\overline{a}$ and $\overline{b}$ denote the means of all literal and clause gates, respectively. The $\ell_1$-style terms encourage sparse rule structures, while the $a(1-a)$ and $b(1-b)$ terms push gates away from $0.5$ and hence toward near-binary decisions. 

We convert the soft rule system into discrete CNF rules by thresholding clause gates and literal gates. For each rule, it keeps up to a configured number of clauses and literals whose gates exceed the thresholds.

\subsection{Recasting Node Embedding Readout as Rule Embedding Readout}
\label{sec:recast_readout}
\paragraph{Rule embedding calculation.}
We estimate the rule embedding by a soft attention readout similar to~\cite{li2016gated} based on the gate activation values in the learned CNF rules. 
We  define the literal embedding by the corresponding global concept embedding with a sign determined by literal polarity:
\begin{equation}
e_\ell = s_\ell \, c_{\kappa(\ell)},
\qquad
s_\ell \in \{+1,-1\},
\end{equation}
where $\kappa(\ell)$ is the global concept referenced by literal $\ell$, $c_{\kappa(\ell)}$ is its global concept embedding learned in~\cref{sec:global_concept}, $s_\ell=+1$ for a positive literal, and $s_\ell=-1$ for a negative literal. 

Following the CNF relaxation, we calculate the clause and rule embedding. For a clause $c$ with retained literals $\gL_{r,c}$, we define the clause embedding as a normalized mixture of literal embedding: 
\begin{equation}
    e_{r,c}=\sum_{\ell \in \gL_{r,c}} \omega_{r,c,\ell}\ e_\ell, 
    \qquad \omega_{r,c,\ell}=\frac{a_{r,c,\ell}}{\sum_{\ell' \in \gL_{r,c}} a_{r,c,\ell'}}.
\end{equation}
Similarly, for a rule $r$ with clauses $\gC_r$, the rule embedding is
\begin{equation}
    e_r=\sum_{c \in \gC_r} \hat{\nu}_{r,c} e_{r,c},
    \qquad \hat{\nu}_{r,c}=\frac{|\gC_r|\, b_{r,c}}{\sum_{c' \in \gC_r} b_{r,c'}}. 
\end{equation}
It multiplies the normalized mixture by the number of retained clauses so that conjunction is not averaged away. The reason is direct: inside a clause, OR represents alternative ways to realize one semantic requirement, so normalization is desirable; across clauses, AND represents jointly required conditions, so the rule embedding should accumulate clause evidence rather than collapse it into a unit-mass average. 

After the discrete rules are fixed, for graph $g$, clause $c$, and rule $r$, the clause score $\tilde{s}_{g,r,c}$ and the rule activation soft mask $\tilde{q}_{g,r}$ are calculated by:
\begin{equation}
    \tilde{s}_{g,r,c}=\max_{\ell \in \gL_{r,c}} \left(\omega_{r,c,\ell} x_{g,\ell}\right),
    \qquad
    \tilde{q}_{g,r}=\prod_{c \in \gC_r}\left[1 - b_{r,c}\left(1-\tilde{s}_{g,r,c}\right)\right].
\end{equation}
The rule embedding is not meant to re-encode Boolean semantics by itself. It is a continuous summary of a discrete rule whose actual firing behavior is still determined by the CNF structure.

\paragraph{Merging class-wise rules into a unified shared rule bank.}
The class-specific CNF rule learner serves only as a candidate generator. To reproduce the full multiclass output of the GNN, we merge all class-wise rule candidates into a unified shared rule bank. 
This two-stage design is deliberate. Learning one shared symbolic rule bank directly from multiclass supervision can make rule discovery much harder and tends to favor generic rules over class-discriminative ones. By first discovering rules in a one-vs-rest manner, each class gets a focused search problem and yields more discriminative candidate rules. We merge them only afterward, so the final shared bank preserves both cross-class diversity and explicit symbolic structure.

\paragraph{Rule-embedding readout via learned rule weights.}
The extracted rules are not assumed to be equally important. Instead, we learn a single global weight vector $w \in \mathbb{R}^{R_{\text{shared}}}$ over the merged discrete rule bank. 
Let $E^{\text{rule}} \in \mathbb{R}^{R_{\text{shared}} \times d_r}$ be the matrix of shared rule embeddings and let $\tilde{q}_g \in [0,1]^{R_{\text{shared}}}$ be the corresponding rule activation scores for graph $g$. We apply thresholded-soft activation:
\begin{equation}
    z_{g,r}=\mathbf{1}\!\left[\tilde{q}_{g,r} \ge \tau_r\right]\tilde{q}_{g,r},
\end{equation}
where $\tau_r$ is the rule-specific threshold inherited from CNF extraction.

The rule-weighted representation is then passed through the frozen classifier head $\phi(\cdot)$ inherited from the base GNN, for the rule-based raw prediction logits: 
\begin{equation}
    \hat{\vo}_g = \phi\!\left(\left(w \odot z_g\right) E^{\text{rule}}\right).
\end{equation}
This can be viewed as replacing the original node-to-graph readout with a weighted rule-to-graph readout: local subgraph concepts are first abstracted into CNF rules, and the weighted combination of active rule embeddings is then fed into the frozen classifier head inherited from the base GNN, where $\hat{\vo}_g \in \mathbb{R}^{C}$ denotes the reconstructed class-logit vector.
Let $\vo_g^{\text{GNN}}$ be the raw prediction logits of the original GNN. The rule-weight objective is to make the unified rule explainer mimic these logits:
\begin{equation}
    \gL_{\text{weight}}
    = \operatorname{MSE}\!\left(\hat{\vo}_g,\; \vo_g^{\text{GNN}}\right)
      + \lambda_w \|w\|_2^2.
\end{equation}
This stage changes the semantics of rule weighting. The learned scalar $w_r$ is global, shared across all classes, and therefore describes how strongly rule $r$ contributes to the model-level readout rather than to only one class-specific decision proxy. Because negative literal embeddings are explicitly defined as the negated embeddings of their positive counterparts, a negative rule weight can be interpreted as emphasizing the negation of that rule. 

\subsection{Test-Time Rule Grounding}
\label{sec:grounding}

At test time, we reuse the global concept bank, the extracted rules, and the learned rule weights. For each unseen graph, we run the same projection pipeline as in training: we extract local concepts, map them to the fixed global concepts, build the graph-level literal vector $x_g$ following~\cref{eq:literal_value}, and evaluate the fixed rules on this graph.
This gives an instance-specific rule explanation: we know which literals are satisfied, which rules are activated, and which activated rules matter most under the learned rule weights.
A key advantage of this instantiation is that it captures both positive and negative evidence. 
For positive literals $P[j]$, we can further map the instantiated rules back to concrete local concept occurrences and their edge sets on the current graph. Negative literals $\neg P[j]$ are visualized as the absence of mode-based global subgraph concepts.

Our framework also enables a direct contribution analysis. For any rule $r$, we can ablate it by zeroing out only its contribution and recomputing the rule-based prediction. Let $\hat{\vo}_g^{(/ r)}$ denote the weighted-rule logits after ablating rule $r$. Then
$\Delta_{g,r} = \hat{\vo}_g  - \hat{\vo}_g^{(/r)}$
measures the class-wise influence of that rule within the unified rule explainer. When the rule explainer faithfully reproduces the original GNN's output, this ablation profile supports analysis on how strongly the rule affects each class and how the reconstructed prediction changes when the rule contribution is removed.

\section{Experiments and Analysis} \label{sec:experiment}
We conduct extensive experiments on a broad collection of both synthetic and real-world datasets to investigate the following research questions:
\begin{enumerate}
    \item[\textbf{RQ1.}] How faithfully can {\mymethod} reproduce the raw multiclass logits of the base GNN?
    \item[\textbf{RQ2.}] Do the critical rules identified by {\mymethod} support the base GNN's prediction? 
    \item[\textbf{RQ3.}] Are the key design components of {\mymethod} necessary for its performance?
    \item[\textbf{RQ4.}] How does {\mymethod} compare with existing class-wise rule-based GNN explainers?
\end{enumerate}

\paragraph{Datasets.}
We use five benchmark datasets, summarized in \cref{tab:data_stats}~(\underline{\bf Appendix~\ref{app:exp_setup}}). While BA-2Motifs~\cite{luo2020parameterized} and BAMultiShapes~\cite{azzolin2022global} are synthetic datasets, Mutagenicity~\cite{kazius2005derivation} and NCI1~\cite{wale2008comparison, pires2015pkcsm} are real-world molecular datasets. 
We further construct a four-class synthetic dataset, BA-Neg, to evaluate {\mymethod} beyond binary classification. 
BA-Neg contains classes defined not only by the presence of discriminative motifs, but also by the explicit absence of certain concepts, allowing us to test whether {\mymethod} can recover such negative evidence from the base GNN~\footnote{The BA-Neg dataset and the code for constructing it are included in the code package.}.


\subsection{Logit Faithfulness and Prediction Probability Fidelity}
\label{sec:exp:logit_fidelity}
To answer RQ1, we evaluate the Mean-Absolute Error (MAE) between GNN raw multiclass logits and TreeX multiclass logits. We additionally define the \textit{Absolute Probability Fidelity} ($\operatorname{APF}$) and \textit{Target-class Absolute Probability Fidelity} ($\operatorname{TAPF}$) to present how faithful TreeX is in reproducing the base GNN's prediction:
\begin{equation}
    \operatorname{APF}=1-\frac{1}{2|\gD|}\sum_{i=1}^{|\gD|} \Big\|\vp_i^{\operatorname{GNN}} - \vp_i^{\operatorname{rule}} \Big\|_1, 
    \quad
    \operatorname{TAPF}=1-\frac{1}{|\gD|}\sum_{i=1}^{|\gD|} \Big|\vp_i^{\operatorname{GNN}}(y_i) - \vp_i^{\operatorname{rule}}(y_i) \Big|,
\end{equation}
where $\vp_i = \operatorname{softmax}\left(\vo_i\right)$ refers to the prediction probability vector, $y_i$ is the GNN's prediction, $\gD$ is the test set. 
In binary classification, $\operatorname{APF}$ and $\operatorname{TAPF}$ are equivalent. In multiclass tasks such as BA-Neg, however, the two metrics are complementary: $\operatorname{TAPF}$ directly measures confidence preservation on the GNN-predicted class, whereas $\operatorname{APF}$ measures full-distribution agreement across all classes. Due to space constraints, results of $\operatorname{TAPF}$ are moved to \underline{\bf Appendix~\ref{app:delta_logit_tapf}}. 

\begin{table}[t]
    \vspace{-0.2in}
    \caption{Logit-level fidelity and component ablations of {\mymethod}. We report MAE between raw logits and APF (\%) between probability outputs of the base GNN and {\mymethod} (RQ1).
    We compare the full {\mymethod} with degraded variants, showing whether each design is necessary (RQ3). 
    These metrics are not reported for prior explainers because they do not reconstruct GNN multiclass logits.   
    The best results are in \colorbox{red!30}{\textbf{bold}}. {\mymethod} can get the \underline{\bf same} results as the \textit{w/ DNF} variant when let \# literals=1. }
    \label{tab:fidelity}
    \centering
    \scalebox{0.65}{
    \begin{tabular}{lcccccccccc}
    \toprule
    \multirow{2.5}{*}{Variant} & \multicolumn{2}{c}{BA-2Motifs} & \multicolumn{2}{c}{BAMultiShapes} & \multicolumn{2}{c}{BA-Neg} & \multicolumn{2}{c}{Mutagenicity} & \multicolumn{2}{c}{NCI1} \\
    \cmidrule(r){2-3}\cmidrule(r){4-5}\cmidrule(r){6-7}\cmidrule(r){8-9}\cmidrule(r){10-11}
    & MAE $(\downarrow)$ & APF $(\uparrow)$ & MAE $(\downarrow)$ & APF $(\uparrow)$ & MAE $(\downarrow)$ & APF $(\uparrow)$ & MAE $(\downarrow)$ & APF $(\uparrow)$ & MAE $(\downarrow)$ & APF $(\uparrow)$  \\
    \midrule
    w/o rules &  {1.37}\footnotesize{$\pm$0.41} & {63.5}\footnotesize{$\pm$18.1} & {6.42}\footnotesize{$\pm$0.14} & {55.6}\footnotesize{$\pm$1.8} & {2.21}\footnotesize{$\pm$0.07} & {50.9}\footnotesize{$\pm$11.5} & {1.52}\footnotesize{$\pm$0.10} & {67.8}\footnotesize{$\pm$1.2} & {2.90}\footnotesize{$\pm$0.06} & {63.8}\footnotesize{$\pm$2.6} \\
    w/o merge &  {0.90}\footnotesize{$\pm$0.75} & {92.3}\footnotesize{$\pm$3.9} & {5.90}\footnotesize{$\pm$0.56} & {64.0}\footnotesize{$\pm$5.4} & {0.94}\footnotesize{$\pm$0.03} & {84.0}\footnotesize{$\pm$2.7} & {1.20}\footnotesize{$\pm$0.05} & {77.4}\footnotesize{$\pm$1.0} & {2.57}\footnotesize{$\pm$0.07} & {70.2}\footnotesize{$\pm$0.3} \\
    w/o $\neg$ &  {1.29}\footnotesize{$\pm$1.26} & {94.8}\footnotesize{$\pm$1.6} & {4.66}\footnotesize{$\pm$0.56} & {71.4}\footnotesize{$\pm$5.2} & {1.17}\footnotesize{$\pm$0.11} & {71.6}\footnotesize{$\pm$4.1} & {6.68}\footnotesize{$\pm$4.07} & {80.2}\footnotesize{$\pm$2.2} & {9.57}\footnotesize{$\pm$2.80} & {73.0}\footnotesize{$\pm$2.2} \\
    w/ DNF &  {0.28}\footnotesize{$\pm$0.03} & {96.9}\footnotesize{$\pm$0.4} & {4.00}\footnotesize{$\pm$2.41} & {74.9}\footnotesize{$\pm$4.3} & {1.05}\footnotesize{$\pm$0.11} & {91.6}\footnotesize{$\pm$7.2} & \colorbox{red!30}{{\bf 1.05}\footnotesize{$\pm$0.19}} & \colorbox{red!30}{{\bf 82.1}\footnotesize{$\pm$1.2}} & {2.41}\footnotesize{$\pm$0.13} & {74.1}\footnotesize{$\pm$1.4} \\
    \midrule
    {\mymethod} (Ours) & \colorbox{red!30}{{\bf 0.26}\footnotesize{$\pm$0.01}} & \colorbox{red!30}{{\bf 97.2}\footnotesize{$\pm$0.3}} & \colorbox{red!30}{{\bf 3.30}\footnotesize{$\pm$1.48}} & \colorbox{red!30}{{\bf 80.0}\footnotesize{$\pm$3.4}} & \colorbox{red!30}{{\bf 0.90}\footnotesize{$\pm$0.57}} & \colorbox{red!30}{{\bf 92.3}\footnotesize{$\pm$6.9}} & {1.23}\footnotesize{$\pm$0.48} & {79.3}\footnotesize{$\pm$7.2} & \colorbox{red!30}{{\bf 2.30}\footnotesize{$\pm$0.12}} & \colorbox{red!30}{{\bf 75.5}\footnotesize{$\pm$1.4}}\\
    \bottomrule
    \end{tabular}}
    \vspace{-0.1in}
\end{table}

We focus on the full {\mymethod} results in the last row of \cref{tab:fidelity}. 
{\mymethod} consistently reconstructs the base GNN outputs with high APF.
The strongest results appear on BA-2Motifs and BA-Neg, where the learned rule-level readout closely matches the GNN's output distribution. 
The multiclass BA-Neg result is particularly important, as it shows that {\mymethod} is not limited to binary class-wise rule prediction but can reconstruct a multiclass logit-producing decision process. 
On the more challenging real-world molecular datasets, the APF remains above \(75\%\), indicating that the symbolic readout preserves substantial probability-level agreement with the base GNN. 
Since raw-logit MAE depends on the scale of the GNN logits, we interpret MAE together with APF: the former measures direct logit reconstruction error, while the latter measures how much of the GNN's predictive distribution is preserved after softmax. 
Together, these results demonstrate that {\mymethod} provides a faithful symbolic approximation of the base GNN's prediction logits. 
Due to space constraints, qualitative analysis of {\mymethod} is moved to \underline{\bf Appendix~\ref{app:gc_my}--\ref{app:test_ground}}.

\subsection{Rule Impact Analysis}
\label{sec:exp:rule_impact}

\begin{figure*}[t]
  \centering
  \includegraphics[width=\textwidth]{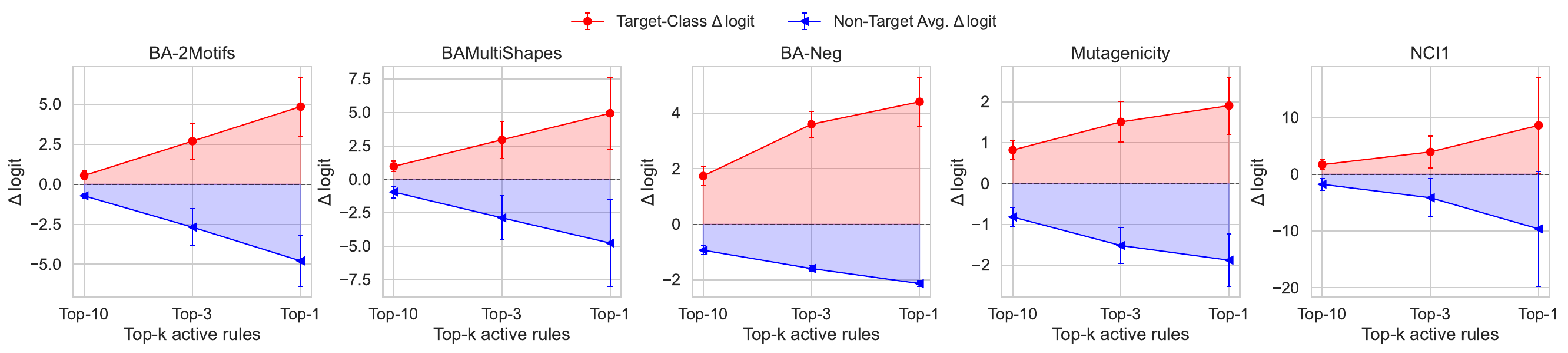}
  \vspace{-0.2in}
  \caption{Top-$k$ rule ablation on the test set. For each dataset, we report the change in reconstructed logits after individually ablating the top-$k$ critical rules and averaging their effects (RQ2). The corresponding numerical values are provided in \cref{tab:rule_ablation}.}
  \vspace{-0.1in}
  \label{fig:rule_ablation}
\end{figure*}

As described in \cref{sec:grounding}, we can use $\Delta_{g,r} = \hat{\vo}_g  - \hat{\vo}_g^{(/r)}$ to measure the class-wise influence of each rule. 
To answer RQ2, we perform active-only rule ablation on the test set.
For each test graph, we individually ablate each of the top-\(k\) active critical rules and average their effects. 
As shown in \cref{fig:rule_ablation}, the identified critical rules consistently support the target class and suppress non-target classes. 
A positive target-class $\Delta$ logit means that removing the rule lowers the target-class logit, while a negative non-target average $\Delta$ logit means that removing the rule raises the average non-target logit. 
Thus, the critical rules actively push the reconstructed prediction toward the target class.

The effect is strongest for Top-1 rules, which is expected because the highest-ranked active rule should be the most discriminative rule for the current graph. 
As \(k\) increases from 1 to 10, the average effect decreases because less critical active rules are included. 
This trend shows that test-time rule instantiation can identify rules with meaningful class-wise influence. 
Overall, these results demonstrate that the critical rules are not merely post-hoc symbolic descriptions; they are functional units whose removal directly changes the reconstructed logits. 
Visualizations are in \underline{\bf Appendix~\ref{app:test_ground}}.

\subsection{Ablation Study}
\label{sec:exp:ablation}
To answer RQ3, we use the degraded variants in \cref{tab:fidelity} to assess whether the main design choices of {\mymethod} are necessary. 
The variant \textit{w/o rules} removes the rule learning module and directly learns a concept-level readout over the global concept bank from \cref{sec:global_concept}. 
This consistently leads to the weakest APF on all datasets, which confirms that simply weighting global concepts is not sufficient. 

The \textit{w/o merge} variant directly learns a shared CNF set instead of first discovering class-wise CNF candidates and then merging them into a shared rule bank. 
Its performance is consistently below the full model. 
This supports our two-stage design: class-wise rule learning provides a more focused search space for discovering discriminative symbolic structures, while the subsequent merge forms a shared rule bank for multiclass logit reconstruction.

The \textit{w/o \(\neg\)} variant disables negative literals and negative rules. 
Its degradation is particularly visible in logit MAE on Mutagenicity and NCI1, where MAE increases from \(1.23\) to \(6.68\) and from \(2.30\) to \(9.57\), respectively. 
This suggests that absence-based evidence is important for reconstructing the raw logits, even when probability-level APF can sometimes remain competitive after softmax. 
The result is also consistent with the motivation of BA-Neg: negative evidence is not merely a visualization artifact, but a useful component of the symbolic readout.

Finally, the \textit{w/ DNF} variant collapses our weighted CNF formulation to a DNF-style rule system by restricting each clause to contain only one literal. 
This variant is a degenerate case of {\mymethod} rather than an equally expressive alternative under the same rule budget. 
Its competitive performance on some datasets, and its best result on Mutagenicity, suggests that the learned decision structure there may be sufficiently simple or DNF-like. 
Rather than weakening our design, this result shows that the proposed weighted CNF formulation can gracefully recover DNF-style rules when they are adequate. 
On the remaining datasets, however, the full model performs better, indicating that OR-of-CNF rules provide additional structural expressiveness for modeling alternatives within required constraints, particularly in compositional or multiclass decision settings. 
Importantly, DNF can be obtained by reducing literals per clause, while a native DNF design cannot naturally recover the richer weighted CNF structure.

Overall, the ablation results support the necessity of the main TreeX design choices. 
Note that subtree-inspired concept extraction is the representation interface of {\mymethod}: it produces the global concept embeddings used in rule embeddings, the concept occurrence masks used for rule activation, and the local edge-level evidence used for test-time grounding. Replacing it would change the concept bank, the literal space, and the test-time grounding procedure, making the resulting system a different explainer rather than a controlled ablation. We therefore do not ablate this component in isolation. 

\subsection{Comparison with Class-Wise Rule-Based Explainers}
\label{sec:exp:compare_baseline}

\begin{table}[t]
\vspace{-0.35in}
    \caption{Prediction Agreement (PA) between the base GNNs and rule-based explainers on the test sets (RQ4) and runtime (in $10^3$ seconds) for various methods. PA is equivalent to data-grounded fidelity in LogicXGNN. Baseline results are reported as presented in~\cite{geng2026logicxgnn}. The best results are in \colorbox{red!30}{\textbf{bold}}, runner-ups are \colorbox{blue!20}{\underline{underlined}}.}
    \label{tab:agreement}
    \centering
    \scalebox{0.75}{
    \begin{tabular}{lcccccc}
    \toprule
    \multirow{2.5}{*}{Method} & \multicolumn{2}{c}{BAMultiShapes} & \multicolumn{2}{c}{Mutagenicity} & \multicolumn{2}{c}{NCI1} \\
    \cmidrule(r){2-3}\cmidrule(r){4-5}\cmidrule(r){6-7}
    & PA $(\uparrow)$ & Time $(\downarrow)$ & PA $(\uparrow)$ & Time $(\downarrow)$ & PA $(\uparrow)$ & Time $(\downarrow)$ \\
    \midrule
    GLG~\cite{azzolin2022global} & {31.09}\footnotesize{$\pm$5.81} & {0.31}\footnotesize{$\pm$0.02} & {38.98}\footnotesize{$\pm$3.01} & {0.73}\footnotesize{$\pm$0.02} & {9.61}\footnotesize{$\pm$7.76} & {0.88}\footnotesize{$\pm$0.02} \\
    GTrail~\cite{armgaan2024graphtrail} & {79.82}\footnotesize{$\pm$3.64} & {2.54}\footnotesize{$\pm$0.12} & {65.93}\footnotesize{$\pm$3.83} &{20.05}\footnotesize{$\pm$1.12} & {60.04}\footnotesize{$\pm$4.91} & {24.07}\footnotesize{$\pm$1.12}\\
    LogicXGNN~\cite{geng2026logicxgnn} & \colorbox{blue!20}{\underline{{82.67}\footnotesize{$\pm$0.57}}} & \colorbox{blue!20}{\underline{{0.02}\footnotesize{$\pm$0.00}}} & \colorbox{red!30}{{\bf 81.36}\footnotesize{$\pm$2.12}} & \colorbox{blue!20}{\underline{{0.61}\footnotesize{$\pm$0.01}}} & \colorbox{blue!20}{\underline{{73.81}\footnotesize{$\pm$2.26}}} & \colorbox{blue!20}{\underline{{0.44}\footnotesize{$\pm$0.00}}} \\
    \midrule
    {\mymethod} (Ours) &\colorbox{red!30}{{\bf 89.00}\footnotesize{$\pm$4.55}} &\colorbox{red!30}{{\bf 0.01}\footnotesize{$\pm$0.00}} & \colorbox{blue!20}{\underline{{80.38}\footnotesize{$\pm$8.41}}} & \colorbox{red!30}{{\bf 0.03}\footnotesize{$\pm$0.00}} &\colorbox{red!30}{{\bf 74.21}\footnotesize{$\pm$2.15}} &\colorbox{red!30}{{\bf 0.03}\footnotesize{$\pm$0.00}} \\
    \bottomrule
    \end{tabular}}
    \vspace{-0.15in}
\end{table}

Since prior rule-based explainers do not reconstruct multiclass logits, we compare {\mymethod} with them using Prediction Agreement (PA), the metric shared by all methods:
\begin{equation}
    \operatorname{PA}=\frac{1}{|\gD|}\sum_{i=1}^{|\gD|}\mathbf{1}\left[\hat{y}_i^{\operatorname{rule}}=y_i\right].
\end{equation}
PA is equivalent to the data-grounded fidelity used in LogicXGNN~\cite{geng2026logicxgnn}, but we use the name Prediction Agreement to distinguish it from our data-grounded probability-level fidelity metrics APF and TAPF.

As shown in \cref{tab:agreement}, {\mymethod} is competitive with existing class-wise rule-based explainers under this shared evaluation metric, while being substantially faster. 
On BAMultiShapes, {\mymethod} achieves the highest PA, improving over LogicXGNN from \(82.67\%\) to \(89.00\%\), and reduces runtime from \(0.02\) to \(0.01\) in \(10^3\) seconds. 
On NCI1, {\mymethod} also slightly improves PA over LogicXGNN while reducing runtime from \(0.44\) to \(0.03\), a speedup of about \(14.7\times\). 
On Mutagenicity, the results of {\mymethod} are comparable with the best baseline, and {\mymethod} reduces runtime from \(0.61\) to \(0.03\), yielding about a \(20.3\times\) speedup.
This is consistent with our DNF ablation: Mutagenicity appears to be well captured by simpler DNF-like class-wise rules, whereas BAMultiShapes contains more substitutable motif combinations and benefits more from our OR-of-CNF structure and shared rule bank.

These results should be interpreted together with the scope of each method. 
GLGExplainer relies on auxiliary local explanations before extracting global formulas, GraphTrail uses costly concept-mining heuristics and outputs class-wise formulas, and LogicXGNN extracts discrete DNF-style rules over pre-enumerated local topology catalog. 
None of these methods produces rule embeddings, a shared weighted rule bank, or reconstructed multiclass logits. 
{\mymethod} remains competitive under this traditional hard-label agreement metric while additionally supporting rule-level logit reconstruction, learned rule weights, and rule-level ablation.
Qualitative comparison is in \underline{\bf Appendix~\ref{app:gc_baselines}}.

\section{Conclusion, Limitations and Future Directions}
In this paper, we introduce {\mymethod}, a model-level GNN explanation framework that recasts graph-level readout as symbolic rule-level readout. 
Instead of extracting class-wise rules that only predict hard labels, {\mymethod} reconstructs the base GNN's raw multiclass logits from a shared bank of grounded logical rules. 
This design yields explanations that are global, test-time instantiable, subgraph-grounded, and analyzable through rule-level weights and ablations.

A current limitation is that rule embeddings are computed by directly aggregating concept embeddings according to the learned rule structure. 
This makes the explanation stable and transparent, but may be less expressive for domains where concepts interact through more complex semantics. 
Extending rule-level readout reconstruction to broader graph tasks and other neural architectures is therefore an important direction for future work.
While this work focuses on GNNs, it suggests a broader route for explaining neural models with structured intermediate representations: identify human-interpretable concepts, compose them into symbolic rules, and use these rules as readout-level units to reconstruct the model's logits.




\bibliography{custom}
\bibliographystyle{plain}


\appendix
\section{Theoretical Justification for Subtree-Based Subgraph Concept Extraction.}
\label{app:just_subtree}
The local concept mining mechanism we introduced in \cref{sec:local_concept} utilizes two critical designs: i) mining over \emph{subtrees} rather than subgraphs to extract concepts; ii) representing the full $L$-hop subtrees by the $L$-th layer root node embeddings. 

The first design sharply reduces the search space compared with mining over all possible subgraphs. This is because in each graph instance with $N$ nodes, there are exactly \(N\) full \(L\)-hop rooted subtrees, whereas the number of possible subgraphs can grow exponentially with \(N\).
Although searching over possible subgraphs in a single instance is feasible in local-level explainability~\cite{yuan2021explainability, zhang2022gstarx, shan2021reinforcement}, it becomes much more challenging in global-level explainability, where the dataset can be large. Our method, mining over subtrees, provides a practical direction for extracting global graph concepts. 

The second design further improves the mining process compared with existing explainability methods~\cite{yuan2021explainability, azzolin2022global}. Typically, existing methods represent subgraphs by feeding them into the original GNN to obtain subgraph embeddings, necessitating additional calculations for each possible subgraph. In contrast, we directly use the node embeddings to represent subtrees, which are easily obtained by feeding the graph to the GNN just once. In the remainder of this section, we explain why we can use root node embeddings to represent subtrees, by showing that the $l$-th layer root node embedding from a maximally powerful MPGNN is an exact mapping of the corresponding full $l$-hop rooted subtree. 

\subsection{Preliminaries}
\paragraph{WL algorithm.}
Please see \cref{alg:wl} for the review of the WL algorithm. 
\begin{algorithm}[ht]
    \caption{The $1$-dimensional Weisfeiler-Lehman Algorithm}
    \label{alg:wl}
    \begin{algorithmic}[1]
        \STATE {\bfseries Input:} Graph $G=(\mathcal{V,E})$, the number of iterations $T$
        \STATE {\bfseries Output:} Color mapping $\mathcal{X}_G:\mathcal{V}\rightarrow\mathcal{C}$
        \STATE {\bfseries Initialize:} $\mathcal{X}_G(v) \coloneqq \text{hash}(G[v])$ for all ${v}\in \mathcal{V}$
        \FOR{$t\leftarrow 1$ {\bfseries to} $T$}
            \FOR{{\bfseries each} ${v}\in\mathcal{V}$}
            \STATE $\mathcal{X}^t_G(v)\coloneqq \text{hash}\left( \mathcal{X}^{t-1}_G(v), \{\!\{  \mathcal{X}^{t-1}_G(u) :u\in \mathcal{N}_G(v)  \}\!\}\right)$
            \STATE {\bfseries break upon convergence}
        \ENDFOR
        \ENDFOR
        \STATE {\bfseries Return:} $\mathcal{X}^{T}_G$
    \end{algorithmic}
\end{algorithm}

\paragraph{Message-Passing Graph Neural Networks.}
In this paper, we focus on WL-based GNNs~\cite{kipf2017semisupervised,xupowerful,hamilton2017inductive} (or MPNNs). 
Typical WL-based GNNs aggregate the information from the 1-hop neighbours $\gN$ of $v$ as:
\begin{equation} \label{eq:gnn}
    \vh^{(l)}_v = \operatorname{UPDATE}^{(l)}\left( \vh^{(l-1)}_v, \operatorname{AGG}^{(l)} \left( \left\{ \vh_u^{(l-1)}: u \in \gN(v)\right\}\right) \right), 
\end{equation}
where $\operatorname{UPDATE}^{(l)}$ and $\operatorname{AGG}^{(l)}$ represent the updating and aggregation functions. In particular, an example maximally powerful MPGNN, GIN~\cite{xupowerful}, updates node representations as:
\begin{equation}
    \label{eq:gin}
    \vh_v^{(l)} = \operatorname{MLP}^{(l)}\left(\left(1+\epsilon^{(l)}\right)\cdot \vh_v^{(l-1)}+\sum_{u\in\gN(v)}\vh_u^{(l-1)} \right).
\end{equation}

\paragraph{Subtrees.}
Given a graph $G=(V,E)$, a \textit{full $l$-hop subtree} $T^{(l)}_v=(V_{T^{(l)}_v}, E_{T^{(l)}_v})$ rooted at $v\in V$, is the entire underlying tree structure within $l$-hop distance from $v$, where $V_{T^{(l)}_v}, E_{T^{(l)}_v}$ are multisets, \textit{i.e.}, a set with possibly repeating elements. Every element in $V_{T^{(l)}_v}$ is also an element in $V$; every element in $E_{T^{(l)}_v}$ is also an element in $E$. The repetitions of the same node in the original graph are treated as distinct nodes in the subtrees, such that the pattern is still a cycle-free tree. Each subtree of $G$ corresponds to a subgraph in the original graph $G$. In the $1$-WL test~\cite{leman1968reduction}, subtrees are used to distinguish the underlying subgraphs. 

\subsection{Explaining Maximally Powerful MPNNs}
\label{app:max_gnn_subtree}
\begin{definition} [\bf Perfect Rooted Tree Representation] \label{def:perfect_tree}
    Let $T_v$ denote a tree in a countable space $\gX$, which is rooted at node $v$, $f(\cdot)$ be the only function to generate the representations of rooted trees in the space, $\vh_v$ be the representation of $T_v$. We define $\vh_v$ be the Perfect Rooted Tree Representation of $T_v$, if the following holds: For any arbitrary same-depth rooted tree $T_u$ in the same countable space, $\,\vh_v=\vh_u$ if and only if $T_v$ is isomorphic to $T_u$.
\end{definition}

We then show that if both $\operatorname{AGG}$ and $\operatorname{UPDATE}$ in \cref{eq:gnn} are injective, then the $l$-th layer node embedding is a Perfect Rooted Tree Representation of the full $l$-hop rooted subtree by mathematical induction, which is described in \cref{thm:gnn_subtree}. 

\begin{theorem}
    \label{thm:gnn_subtree}
    Given a graph $G=(V,E)$ with the countable input node features $\rmX$, and a $L$-layer GNN $f(\cdot)$ that updates the layer-wise node-embeddings by \cref{eq:gnn}. Then $\forall l\in \{1,\dots L\}$ and $\forall v\in V$, the $l$-th layer node embedding $\vh^{(l)}_v$ is a Perfect Rooted Tree Representation of the full $l$-hop subtree rooted at $v$, if the functions $\operatorname{AGG}$ and $\operatorname{UPDATE}$ in \cref{eq:gnn} are injective.
\end{theorem}

\begin{proof}
    Please refer to \cref{app:prove:gnn_subtree} for the proof.
\end{proof}

As presented in \cref{eq:gin}, the maximally powerful MPGNN, e.g., GIN, utilizes add-pooling and MLP as the $\operatorname{AGG}$ and $\operatorname{UPDATE}$ functions, which are both injective for countable inputs. By \cref{thm:gnn_subtree}, we acquire that GIN uniquely maps the full $l$-hop subtrees to the $l$-th layer embeddings of their root nodes. 
Therefore, we can directly use the root node embeddings to represent the corresponding rooted subtrees. 

\subsection{Explaining Less Powerful MPNNs}
\label{app:less_gnns}
The $1$-WL test is able to decide the graph isomorphism in most real-world cases~\cite{zopf20221}. Therefore, for the maximally expressive MPNNs, we induce the subgraph-level concepts with the full $l$-hop subtrees in \cref{sec:method}. However, there exist many GNNs like GCN~\cite{kipf2017semisupervised} and GraphSAGE~\cite{hamilton2017inductive} that are significantly less expressive than the $1$-WL test. In this subsection, we explain how these less expressive GNNs can be explained using our proposed {\smodel}. 

GCN~\cite{kipf2017semisupervised} updates node representations as: 
\begin{equation}
    \label{eq:gcn}
    \vh^{(l)}_v = \operatorname{ReLU}\left( W\cdot\operatorname{MEAN}\left\{ \vh_u^{(l-1)}, \forall u\in \gN(v)\cup\{v\}\right\}\right),
\end{equation}
where $W$ is a learnable matrix, and $\operatorname{MEAN}$ represents the element-wise mean-pooling. 

GraphSAGE~\cite{hamilton2017inductive} updates node representations as the linear mapping on the concatenation of the last-layer node embedding and the aggregation of the neighbouring node embeddings: 
    \begin{align}
    \label{eq:sage}
        \va^{(l-1)}_v &= \text{MAX}\left( \left\{ \sigma\left(W\cdot \vh_u^{(l-1)}\right), \forall u\in \gN(v) \right\} \right), \\
        \vh^{(l)}_v &= \sigma\left(W\cdot \left[\vh_v^{(l-1)}, \va^{(l-1)}_v\right]\right),
    \end{align}
where $\text{MAX}$ represents the element-wise max-pooling, $\sigma$ refers to $\operatorname{ReLU}$.

\subsubsection{Relationship between Subtrees and Node Embeddings for Less Powerful MPNNs}

In this subsection, we discuss the less expressive GNNs by studying the $\operatorname{UPDATE}$ and $\operatorname{AGG}$ functions of them. 

If the $\operatorname{UPDATE}$ function is injective, then the distinctness of the embeddings will not change after feeding into the $\operatorname{UPDATE}$ function. 
Therefore, if a $l$-th layer node embedding after $\operatorname{AGG}$ can map two distinct subtrees at the same time, then after the injective $\operatorname{UPDATE}$ function, it will represent the same pair of distinct subtrees. 
This forms the following corollary.

\begin{corollary} \label{corol:less_expressive}
    Given a graph $G=(V,E)$, let $f(\cdot)$ denote a $L$-layer GNN that updates the countable layer-wise node-embeddings by ~\cref{eq:gnn}. We define 
    the intermediate $l$-th layer embedding derived by the $\operatorname{AGG}$ function of node $v\in V$ as 
    $\vh_{v,\operatorname{AGG}}^{(l)}$, where $\vh_v^{(l)}=\operatorname{UPDATE}^{(l)}\left( \vh_{v,\operatorname{AGG}}^{(l)} \right)$.
    Then the following holds:
    \begin{itemize}
        \item[(i)] If $\operatorname{UPDATE}$ is injective, then $\vh_v^{(l)}$ is a Perfect Rooted Tree Representation of the full $l$-hop subtree rooted at $v$ if and only if $\,\vh_{v,\operatorname{AGG}}^{(l)}$ is a Perfect Rooted Tree Representation of the full $l$-hop subtree rooted at $v$.
        \item[(ii)] If $\operatorname{UPDATE}$ is injective, then $\vh_v^{(l)}$ is a mapping of both a full $l$-hop subtree and another non-isomorphic $l$-hop subtree rooted at $v$, \textit{if and only if} $\vh_{v,\operatorname{AGG}}^{(l)}$ is a mapping of the same pair of non-isomorphic subtrees at the same time. 
    \end{itemize}
\end{corollary}
\begin{proof}
    Please see \cref{app:proof_less_express} for the proof.
\end{proof}

In the family of WL-based GNNs, a 1-layer MLP (with bias term) or a 2-layer MLP is typically used as the $\operatorname{UPDATE}$ function, which are both injective functions. Some variants of GNNs, like GraphSAGE~\cite{hamilton2017inductive}, may additionally utilize a concatenation as shown in \cref{eq:sage}, which is also injective. Therefore, we assume the $\operatorname{UPDATE}$ function in \cref{eq:gnn} of less powerful MPNNs is an injective function. 

Next, we discuss the $\operatorname{AGG}$ function in \cref{eq:gnn}, including the commonly used add-pooling, mean-pooling and max-pooling methods. The most expressive add-pooling is discussed in \cref{app:max_gnn_subtree}. Mean-pooling, as investigated in \cite{xupowerful}, captures the ``distributions'' of elements in a multiset. In other words, there may exist another subtree $T^{(l)\prime}_v$ that contains the same set of unique elements as the full $l$-hop subtree $T^{(l)}_v$, where the distribution of unique elements in $T^{(l)\prime}_v$ is the same as in $T^{(l)}_v$. Such two subtrees will have the same root node embedding if using the mean-pooling, which can be treated as a perfect representation of the \textit{node distribution} in the full $l$-hop rooted subgraphs. GCN is an example of using mean-pooling as shown in \cref{eq:gcn}. 

Max-pooling treats a multiset as a set~\cite{xupowerful}. This means if there exist two subtrees $T^{(l)\prime}_v$ and $T^{(l)}_v$ contain the same set of unique elements, they will have the same root node embedding, which can be treated as a perfect representation of the unique node set in the full $l$-hop rooted subgraphs. GraphSAGE is an example of max-pooling as shown in \cref{eq:sage}. 

It is worth noting that the MPNNs with mean-pooling or max-pooling tend to be less expressive hence are less preferred for most tasks. Therefore, our primary focus in this paper is to explain GNNs that demonstrate expressiveness comparable to the WL-test algorithm. In these GNNs, node embeddings precisely represent the full $l$-hop subtrees. 

\subsubsection{Hash Model to Explain the Less Powerful MPNNs}

As we discussed in \cref{app:less_gnns}, GNNs that update the node embeddings by mean-pooling may produce the same root node embeddings for the subtrees with the same node distribution, and the ones using max-pooling may produce the same root node embeddings for the subtrees with the same node set. For these GNNs, there is a higher risk of clustering multiple entirely different substructures to the same concept as those non-isomorphic subtrees may share the same root node embedding. 

To mitigate this issue, we introduce a \emph{hash model} that aids in distinguishing global graph concepts induced by subtrees with similar node distributions or node sets but significantly different structures. Specifically, after we obtain the local clusters, we feed each local graph concept to a hash model $\Omega(\cdot)$ that returns the graph embedding of it. Then, we concatenate the hashed graph embedding of the local concept to its original embedding to obtain the updated embedding for it. Let $\vm_{o}$ be the original embedding of local concept $o$, then the updated embedding $\hat{\vm_{o}}$ is 
$\hat{\vm_{o}} = \text{Concat}\left(\vm_{o}, \Omega(o)\right). $
The steps afterwards remain the same as we discussed in \cref{sec:method}. A hash model should be able to distinguish graph concepts that have the same node distribution or the same node set. For example, the WL-test can be used as a hash model.

\section{Lemmas and Proofs} \label{app:proofs}

\subsection{Proof of \texorpdfstring{\cref{thm:gnn_subtree}}{Theorem 2}}
\label{app:prove:gnn_subtree}
\begin{proof}
    We prove \cref{thm:gnn_subtree} by mathematical induction. In the base step, we aim to prove \cref{lem:base_step}. In the inductive step, we aim to prove \cref{lemma:induc_step}.
    \begin{lemma} [Base step]
        \label{lem:base_step}
        Given a graph $G=(V,E)$ with the countable input node features $\rmX$, and a $L$-layer GNN $f(\cdot)$ that updates the layer-wise node-embeddings by \cref{eq:gnn}. Then $\forall v\in V$, the first layer node embedding $\vh^{(1)}_v$ is a Perfect Rooted Tree Representation of the $1$-hop subtree rooted at $v$, if the functions $\operatorname{AGG}$ and $\operatorname{UPDATE}$ in \cref{eq:gnn} are injective. 
    \end{lemma}
    \begin{proof}
        Let $T_v$ be the $1$-hop subtree rooted at $v$ in $G$. Assume $\vh^{(1)}_v$ is not a Perfect Rooted Tree Representation of $T_v$. Then either of the cases should hold: 
        \begin{itemize}
            \item[(i)] There exists another 1-hop subtree $T_u$, embedded by $f(\cdot)$ as $\vh_u^{(1)}$, which is non-isomorphic to $T_v$, but $\vh^{(1)}_v=\vh_u^{(1)}$; 
            \item[(ii)] There exists an isomorphic subtree $T_u$, embedded by $f(\cdot)$ as $\vh_u^{(1)}$, where $\vh_v^{(1)} \neq \vh_u^{(1)}$. 
        \end{itemize}
        According to \cref{eq:gnn}, we can calculate $\vh^{(1)}_v$ and $\vh_u^{(1)}$ by: 
        \begin{equation*} \vh^{(1)}_v = \operatorname{UPDATE}^{(1)}\left( \rmX_v, \operatorname{AGG}^{(1)} \left( \left\{ \rmX_w: w \in \gN(v)\right\}\right) \right), \end{equation*}
        \begin{equation*} \label{eq:vhu1}
            \vh_u^{(1)} = \operatorname{UPDATE}^{(1)}\left( \rmX_u, \operatorname{AGG}^{(1)} \left( \left\{ \rmX_w: w \in \gN(u)\right\}\right) \right),
        \end{equation*} 
        First, we consider Case (i). If $T_v$ and $T_u$ are non-isomorphic $1$-hop subtrees, then $\rmX_u\neq\rmX_v$, or the multisets $\left\{ \rmX_w: w \in \gN(v)\right\}\neq\left\{ \rmX_w: w \in \gN(u)\right\}$.
        Recall that an injective function $g(\cdot)$ refers to a function that that maps distinct elements of its domain to distinct elements. That is, $x_1\neq x_2$ implies $g(x_1)\neq g(x_2)$; $x_1= x_2$ implies $g(x_1)= g(x_2)$. 
        If $\rmX_u\neq\rmX_v$ or $\left\{ \rmX_w: w \in \gN(v)\right\}\neq\left\{ \rmX_w: w \in \gN(u)\right\}$, since $\operatorname{AGG}$ and $\operatorname{UPDATE}$ are injective, we have $\vh_v^{(1)}\neq\vh_u^{(1)}$. 
        Hence, we have reached a contradiction. 
        
        Next, we consider Case (ii). 
        If $T_u$ is isomorphic to $T_v$, then $\rmX_u=\rmX_v$ and the multisets $\left\{ \rmX_w: w \in \gN(v)\right\}=\left\{ \rmX_w: w \in \gN(u)\right\}$. Since $\operatorname{AGG}$ and $\operatorname{UPDATE}$ are both injective, we have $\vh_v^{(1)}=\vh_u^{(1)}$. Hence, we have reached a contradiction. 

        Therefore, if the functions $\operatorname{AGG}$ and $\operatorname{UPDATE}$ in \cref{eq:gnn} are injective, $\vh^{(1)}_v$ is a Perfect Rooted Tree Representation of the $1$-hop subtree rooted at $v$. 
    \end{proof}

    \begin{lemma} [Inductive step]
        \label{lemma:induc_step}
        Given a graph $G=(V,E)$, assume the countable node representation $\vh^{(l-1)}_v$ for $v\in V$ be the Perfect Rooted Tree Representation of the corresponding $(l-1)$-hop subtrees rooted at $v$. We calculate the $l$-th layer representation of $v$, i.e., $\vh^{(l)}_v$, using \cref{eq:gnn}. Then $\vh^{(l)}_v$ is a Perfect Rooted Tree Representation of the full $l$-hop subtree rooted at $v$ if the functions $\operatorname{AGG}$ and $\operatorname{UPDATE}$ are injective. 
    \end{lemma}
    \begin{proof}
        Let $T_v^{(l)}$ be the full $l$-hop subtree rooted at $v$ in $G$. Assume $\vh^{(l)}_v$ is not a Perfect Rooted Tree Representation of the full $l$-hop subtree rooted at $v$. Then, either of the following cases should hold:
        \begin{itemize}
            \item[(i)] There exists another full l-hop subtree $T_u^{(l)}$, embedded by $f(\cdot)$ as $\vh_u^{(l)}$, which is non-isomorphic to $T_v^{(l)}$, but $\vh^{(l)}_v=\vh_u^{(l)}$; 
            \item[(ii)] There exists an isomorphic subtree $T_u^{(l)}$, embedded by $f(\cdot)$ as $\vh_u^{(l)}$, where $\vh_v^{(l)} \neq \vh_u^{(l)}$. 
        \end{itemize}
        According to \cref{eq:gnn}, we can calculate $\vh^{(l)}_v$ and $\vh_u^{(l)}$ by: 
        \begin{equation*} \vh^{(l)}_v = \operatorname{UPDATE}^{(l)}\left( \vh^{(l-1)}_v, \operatorname{AGG}^{(l)} \left( \left\{ \vh_w^{(l-1)}: w \in \gN(v)\right\}\right) \right),\end{equation*}
        \begin{equation*} \vh^{(l)}_u = \operatorname{UPDATE}^{(l)}\left( \vh^{(l-1)}_u, \operatorname{AGG}^{(l)} \left( \left\{ \vh_w^{(l-1)}: w \in \gN(u)\right\}\right) \right). \end{equation*}
        \begin{sloppypar}
        First, we consider Case (i). If $T_v^{(l)}$ and $T_u^{(l)}$ are non-isomorphic full $l$-hop subtrees, then the $(l-1)$-hop subtrees $T_v^{(l-1)}$ and $T_u^{(l-1)}$ are non-isomorphic, or the multisets $\left\{ \vh_w^{(l-1)}: w \in \gN(v)\right\}\neq\left\{ \vh_w^{(l-1)}: w \in \gN(u)\right\}$. 
        \end{sloppypar}
        Since $\vh^{(l-1)}_v$ is the Perfect Rooted Tree Representation of the corresponding $(l-1)$-hop subtrees rooted at $v$, we have: If $T_v^{(l-1)}$ and $T_u^{(l-1)}$ are non-isomorphic, then $\vh^{(l-1)}_v\neq\vh^{(l-1)}_u$. Since the functions $\operatorname{AGG}$ and $\operatorname{UPDATE}$ are injective, we have $\vh^{(l)}_v\neq\vh^{(l)}_u$. Hence, we have reached a contradiction. 

        Next, we consider Case (ii). 
        If $T_u^{(l)}$ is isomorphic to $T_v^{(l)}$, then the $(l-1)$-hop subtrees $T_u^{(l-1)}$ and $T_v^{(l-1)}$ are also isomorphic. And we have  and the multisets $\left\{ \vh_w^{(l-1)}: w \in \gN(v)\right\}=\left\{ \vh_w^{(l-1)}: w \in \gN(u)\right\}$. Since $\vh^{(l-1)}_v$ is the Perfect Rooted Tree Representation of the corresponding $(l-1)$-hop subtrees rooted at $v$, we have $\vh^{(l-1)}_v=\vh^{(l-1)}_u$. Since $\operatorname{AGG}$ and $\operatorname{UPDATE}$ are both injective, we have $\vh_v^{(l)}=\vh_u^{(l)}$. Hence, we have reached a contradiction. 

        Therefore, if the functions $\operatorname{AGG}$ and $\operatorname{UPDATE}$ in \cref{eq:gnn} are injective, $\vh^{(l-1)}_v$ for $v\in V$ is the Perfect Rooted Tree Representation of the corresponding $(l-1)$-hop subtrees rooted at $v$, then $\vh^{(l)}_v$ is a Perfect Rooted Tree Representation of the $l$-hop subtree rooted at $v$. 
    \end{proof}
    The following lemma shows that if the input of a GNN is countable, then the node embeddings are also countable.
    \begin{lemma}
        \label{lem:countable}
        \cite{xupowerful}
        Assume the input feature $\gX$ is countable. Let $g^{(l)}$ be the function parameterized by a GNN's $l$-th layer for $l=1,\dots,L$, where $g^{(1)}$ is defined on multisets $X\subset\gX$ of bounded size. The range of $g^{(l)}$, i.e., the space of node hidden features $\vh_v^{(l)}$, is also countable for all $l=1,\dots,L$.
    \end{lemma}
    This lemma implies that if the input $\rmX_v$ for any $v$ is countable, then $\vh_v^{(l)}$ for any $l$ is also countable, making our assumption in \cref{lemma:induc_step} valid. 

    Hence, we have proved \cref{thm:gnn_subtree} using mathematical induction. 
\end{proof}



\subsection{Proof of \texorpdfstring{Corollary~\ref{corol:less_expressive}}{Corollary 3}} \label{app:proof_less_express}
\begin{proof}
    We first prove (i). First, we assume $\vh_{v,\operatorname{AGG}}^{(l)}$ is a Perfect Rooted Tree Representation of the full $l$-hop subtree $T_v$ rooted at $v$. By Definition~\ref{def:perfect_tree}, for any arbitrary same-depth rooted tree $T_u$ in the same countable space, $\vh_{v,\operatorname{AGG}}^{(l)}=\vh_{u,\operatorname{AGG}}^{(l)}$ if and only if $T_v$ is isomorphic to $T_u$. Since $\operatorname{UPDATE}$ is injective, we have for any arbitrary same-depth rooted tree $T_u$ in the same countable space, $\vh_v^{(l)}=\vh_u^{(l)}$ if and only if $T_v$ is isomorphic to $T_u$. Therefore, if $\operatorname{UPDATE}$ is injective and $\vh_{v,\operatorname{AGG}}^{(l)}$ is a Perfect Rooted Tree Representation of the full $l$-hop subtree $T_v$ rooted at $v$, then $\vh_v^{(l)}$ is a Perfect Rooted Tree Representation of the full $l$-hop subtree rooted at $v$. 

    Secondly, we assume $\vh_{v,\operatorname{AGG}}^{(l)}$ is not a Perfect Rooted Tree Representation of the full $l$-hop subtree $T_v$ rooted at $v$. Then the condition in Definition~\ref{def:perfect_tree} is violated, which means that one of the following cases holds:
    \begin{itemize}
        \item $T_v$ and $T_u$ are isomorphic, but $\vh_{v,\operatorname{AGG}}^{(l)}\neq\vh_{u,\operatorname{AGG}}^{(l)}$;
        \item $T_v$ and $T_u$ are non-isomorphic, but $\vh_{v,\operatorname{AGG}}^{(l)}=\vh_{u,\operatorname{AGG}}^{(l)}$.
    \end{itemize}
    Since $\operatorname{UPDATE}$ is injective, either of the following cases holds:
    \begin{itemize}
        \item $T_v$ and $T_u$ are isomorphic, but $\vh_v^{(l)}\neq\vh_u^{(l)}$;
        \item $T_v$ and $T_u$ are non-isomorphic, but $\vh_v^{(l)}=\vh_u^{(l)}$.
    \end{itemize}
    Therefore, if $\operatorname{UPDATE}$ is injective and $\vh_{v,\operatorname{AGG}}^{(l)}$ is not a Perfect Rooted Tree Representation of the full $l$-hop subtree $T_v$ rooted at $v$, then $\vh_v^{(l)}$ is not a Perfect Rooted Tree Representation of the full $l$-hop subtree rooted at $v$. 

    Hence we can conclude that if $\operatorname{UPDATE}$ is injective, then $\vh_v^{(l)}$ is a Perfect Rooted Tree Representation of the full $l$-hop subtree rooted at $v$ if and only if $\,\vh_{v,\operatorname{AGG}}^{(l)}$ is a Perfect Rooted Tree Representation of the full $l$-hop subtree rooted at $v$. 

    Similarly, we prove (ii). First, we assume $\,\vh_{v,\operatorname{AGG}}^{(l)}$ is a mapping of both a full $l$-hop subtree $T^{(l)}_v=(V_{T^{(l)}_v}, E_{T^{(l)}_v})$ and another non-isomorphic $l$-hop subtree $T^{(l)\prime}_v=(V_{T^{(l)\prime}_v}, E_{T^{(l)\prime}_v})$ rooted at $v$.
    Let $\vh_{v,\operatorname{AGG}}^{(l)}$ denote the intermediate $l$-th layer embedding derived by the $\operatorname{AGG}$ function on $T_v^{(l)}$, $\vh_{v,\operatorname{AGG}}^{(l)\prime}$ denote the intermediate $l$-th layer embedding derived by the $\operatorname{AGG}$ function on $T_v^{(l)\prime}$.
    Then we get $\vh_{v,\operatorname{AGG}}^{(l)}=\vh_{v,\operatorname{AGG}}^{(l)\prime}$. Since $\operatorname{UPDATE}$ is injective, $\vh_v^{(l)}=\vh_v^{(l)\prime}$, where $\vh_v^{(l)\prime}$ is the node embedding computed using the same function as $\vh_v^{(l)}$, but on $T^{(l)\prime}_v$. Therefore, if $\operatorname{UPDATE}$ is injective and $\,\vh_{v,\operatorname{AGG}}^{(l)}$ is a mapping of both a full $l$-hop subtree and another $l$-hop subtree rooted at $v$, then $\vh_v^{(l)}$ is a mapping of the same pair of non-isomorphic trees rooted at $v$.
    
    Finally we assume $\,\vh_{v,\operatorname{AGG}}^{(l)}$ is not a mapping of both a full $l$-hop subtree and another $l$-hop subtree rooted at $v$. In other words, $T_v^{(l)}$ and a non-isomorphic subtree $T^{(l)\prime}_v$ always have their distinct representations $\,\vh_{v,\operatorname{AGG}}^{(l)}$ and $\,\vh_{v,\operatorname{AGG}}^{(l)\prime}$, where $\vh_{v,\operatorname{AGG}}^{(l)}\neq\vh_{v,\operatorname{AGG}}^{(l)\prime}$. Since $\operatorname{UPDATE}$ is injective, we have $\vh_v^{(l)}\neq\vh_v^{(l)\prime}$. Therefore, if $\operatorname{UPDATE}$ is injective and $\,\vh_{v,\operatorname{AGG}}^{(l)}$ is not a mapping of both a full $l$-hop subtree and another $l$-hop subtree rooted at $v$, then $\vh_v^{(l)}$ is not a mapping of both a full $l$-hop subtree and another non-isomorphic $l$-hop subtree rooted at $v$ at the same time.

    Hence we conclude that if $\operatorname{UPDATE}$ is injective, then $\vh_v^{(l)}$ is a mapping of both a full $l$-hop subtree and another non-isomorphic $l$-hop subtree rooted at $v$, \textit{if and only if} $\vh_{v,\operatorname{AGG}}^{(l)}$ is a mapping of both a full $l$-hop subtree and another non-isomorphic $l$-hop subtree rooted at $v$ at the same time. 
\end{proof}

\section{Experimental Setup} \label{app:exp_setup}

\subsection{Datasets}
Dataset Statistics are presented in \cref{tab:data_stats}. We evaluate {\mymethod} on five graph classification benchmarks, including three synthetic datasets ({BA-2Motifs}, {BAMultiShapes}, and {BA-Neg}) and two real-world molecular datasets ({Mutagenicity} and {NCI1}).

\begin{table*}[htp]
\caption{Statistics of datasets. Edge counts follow the PyG storage convention, i.e., each undirected edge is counted twice in \texttt{edge\_index}.}
\label{tab:data_stats}
\begin{center}
\scalebox{0.8}{
\begin{tabular}{l|cc|cc|cc|cc|cc}
\toprule
\multirow{2}{*}{\bf Datasets}
& \multicolumn{2}{c|}{\bf BA-2Motifs}
& \multicolumn{2}{c|}{\bf BAMultiShapes}
& \multicolumn{2}{c|}{\bf BA-Neg}
& \multicolumn{2}{c|}{\bf Mutagenicity}
& \multicolumn{2}{c}{\bf NCI1} \\
& \#nodes & \#edges & \#nodes & \#edges & \#nodes & \#edges & \#nodes & \#edges & \#nodes & \#edges \\
\midrule
mean        & 25.0 & 51.0 & 40.0 & 87.5 & 40.0 & 80.9 & 30.3 & 61.5 & 29.9 & 64.6 \\
std         & 0.0  & 1.0  & 0.0  & 7.2  & 3.1  & 6.7  & 20.1 & 33.6 & 13.6 & 29.9 \\
min         & 25   & 49   & 40   & 78   & 35   & 68   & 4    & 6    & 3    & 4    \\
quantile25  & 25   & 50   & 40   & 78   & 37   & 76   & 19   & 38   & 21   & 46   \\
median      & 25   & 50.5 & 40   & 90   & 40   & 80   & 27   & 56   & 27   & 58   \\
quantile75  & 25   & 52   & 40   & 92   & 43   & 86   & 35   & 76   & 35   & 76   \\
max         & 25   & 52   & 40   & 100  & 45   & 94   & 417  & 224  & 111  & 238  \\
\#graphs    & \multicolumn{2}{c|}{1000}
            & \multicolumn{2}{c|}{1000}
            & \multicolumn{2}{c|}{1200}
            & \multicolumn{2}{c|}{4337}
            & \multicolumn{2}{c}{4110} \\
type        & \multicolumn{2}{c|}{synthetic} & \multicolumn{2}{c|}{synthetic} & \multicolumn{2}{c|}{synthetic} & \multicolumn{2}{c|}{real-world} & \multicolumn{2}{c}{real-world} \\
\#classes   & \multicolumn{2}{c|}{2} & \multicolumn{2}{c|}{2} & \multicolumn{2}{c|}{4} & \multicolumn{2}{c|}{2} & \multicolumn{2}{c}{2} \\
feat. dim   & \multicolumn{2}{c|}{10}        & \multicolumn{2}{c|}{10}        & \multicolumn{2}{c|}{10}        & \multicolumn{2}{c|}{14}         & \multicolumn{2}{c}{37} \\
class dist. & \multicolumn{2}{c|}{[500,500]} & \multicolumn{2}{c|}{[500,500]} & \multicolumn{2}{c|}{[300,300,300,300]} & \multicolumn{2}{c|}{[2401,1936]} & \multicolumn{2}{c}{[2053,2057]} \\
\bottomrule
\end{tabular}}
\end{center}
\end{table*}

\begin{figure*}[htp]
  \centering
  \includegraphics[width=\textwidth]{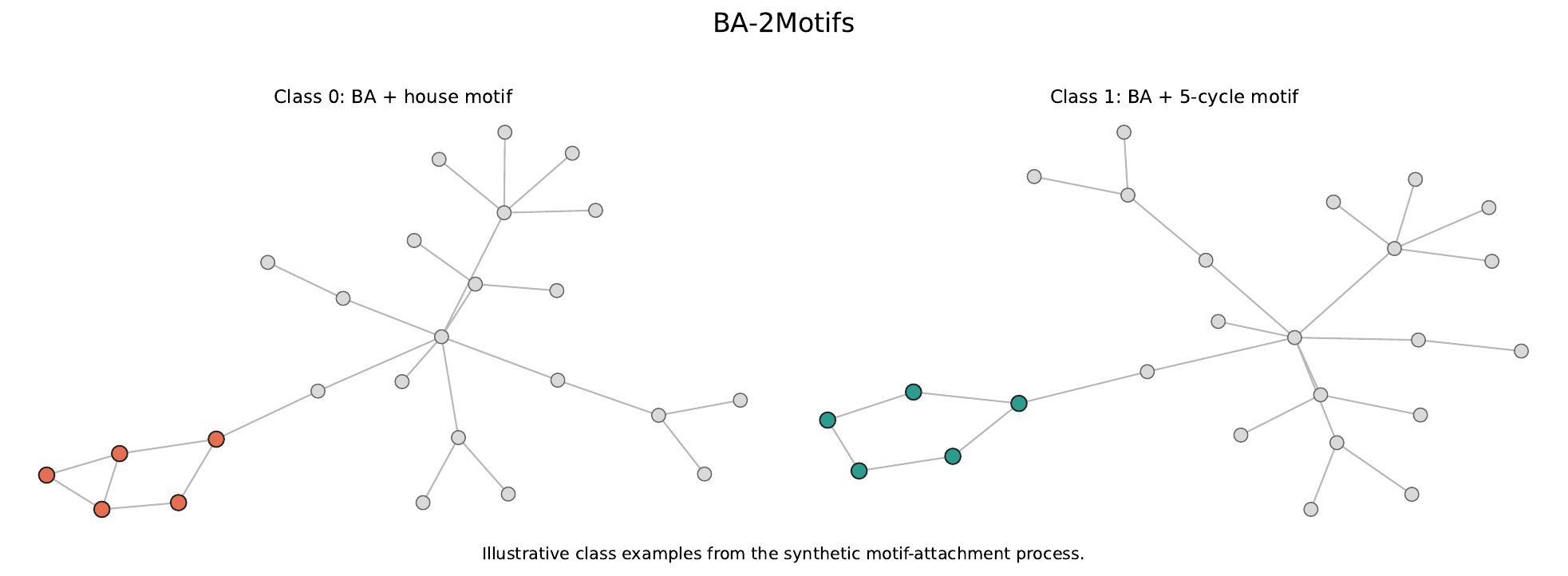}
  \caption{Example graphs of the BA-2Motifs dataset.}
  \label{fig:ba_2motifs_examples}
\end{figure*}

\paragraph{BA-2Motifs.}
As shown in \cref{fig:ba_2motifs_examples}, \textsc{BA-2Motifs} is a synthetic graph classification benchmark in which graphs are generated by attaching either a 5-cycle motif or a house motif to a Barab'asi-Albert (BA) base graph. The task is relatively easy for modern GNNs, and our GIN backbone reaches perfect test accuracy under seed 1234. This dataset is useful because the class signal is controlled and structurally localized, making it well suited for testing whether {\mymethod} can recover concise symbolic rules that align with the ground-truth motif-level decision logic. In particular, it measures whether the learned shared rule bank can faithfully reconstruct a highly clean and separable decision boundary. 

\begin{figure*}[htp]
  \centering
  \includegraphics[width=\textwidth]{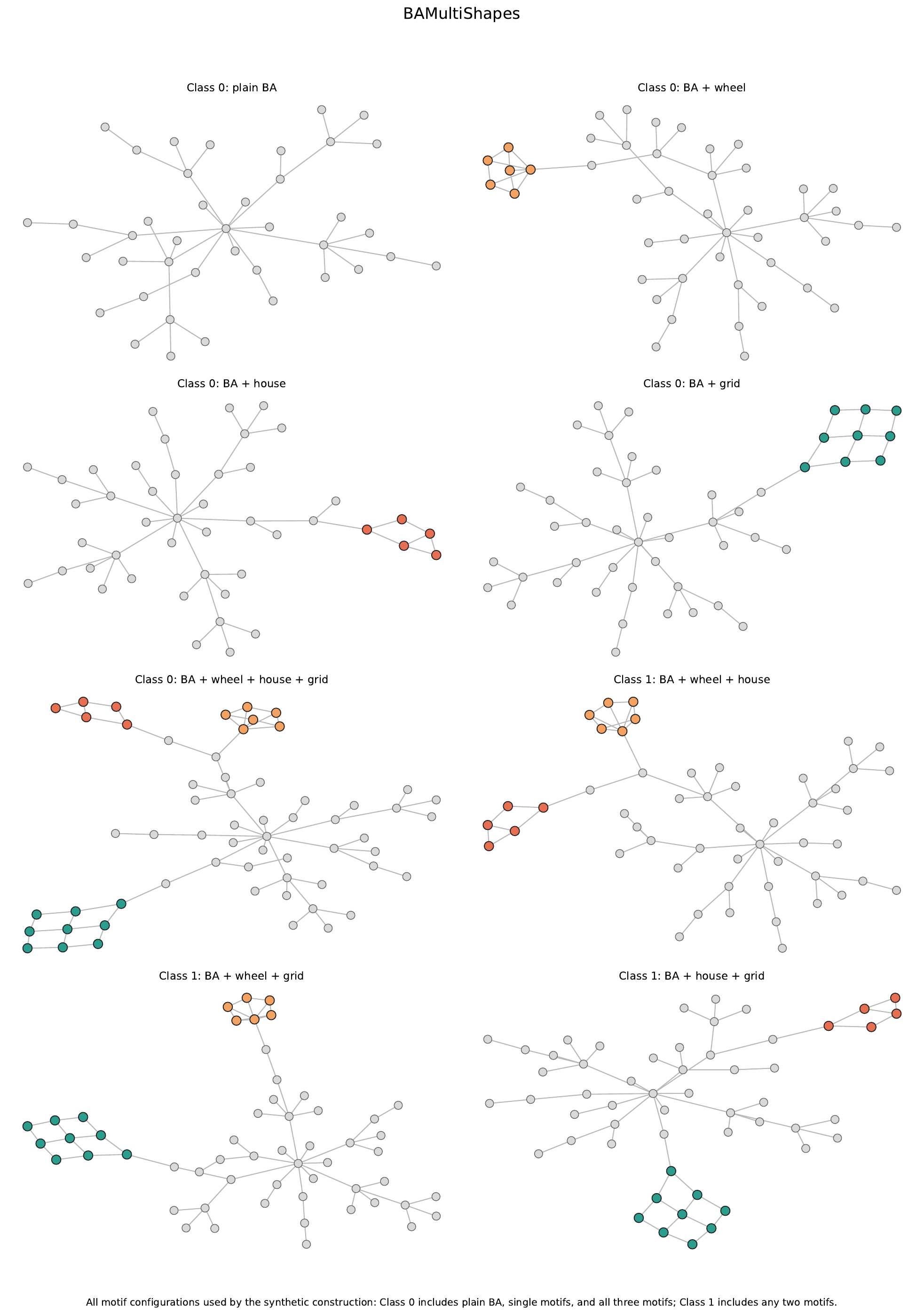}
  \caption{Example graphs of the BAMultiShapes dataset.}
  \label{fig:BAMultiShapes_examples}
\end{figure*}

\paragraph{BAMultiShapes.}
As shown in \cref{fig:BAMultiShapes_examples}, \textsc{BAMultiShapes} is a synthetic benchmark where multiple shape motifs are injected into a common BA-style background graph. 
Specifically, in this dataset, the graphs contain randomly positioned house, grid, and wheel motifs. Class 0 includes plain BA graphs and those with individual motifs or a combination of all three. In contrast, Class 1 comprises BA graphs enriched with any two of the three motifs.
Compared with \textsc{BA-2Motifs}, it contains richer structural variation and more overlap among explanatory substructures, which makes the explanation problem less trivial even when the base GNN remains highly accurate. This dataset is particularly useful for evaluating whether {\mymethod} can organize multiple correlated local concepts into a compact shared rule bank rather than relying on a single dominant pattern. As a result, it tests the compositionality and compactness of the recovered rules.

\begin{figure*}[ht]
  \centering
  \includegraphics[width=\textwidth]{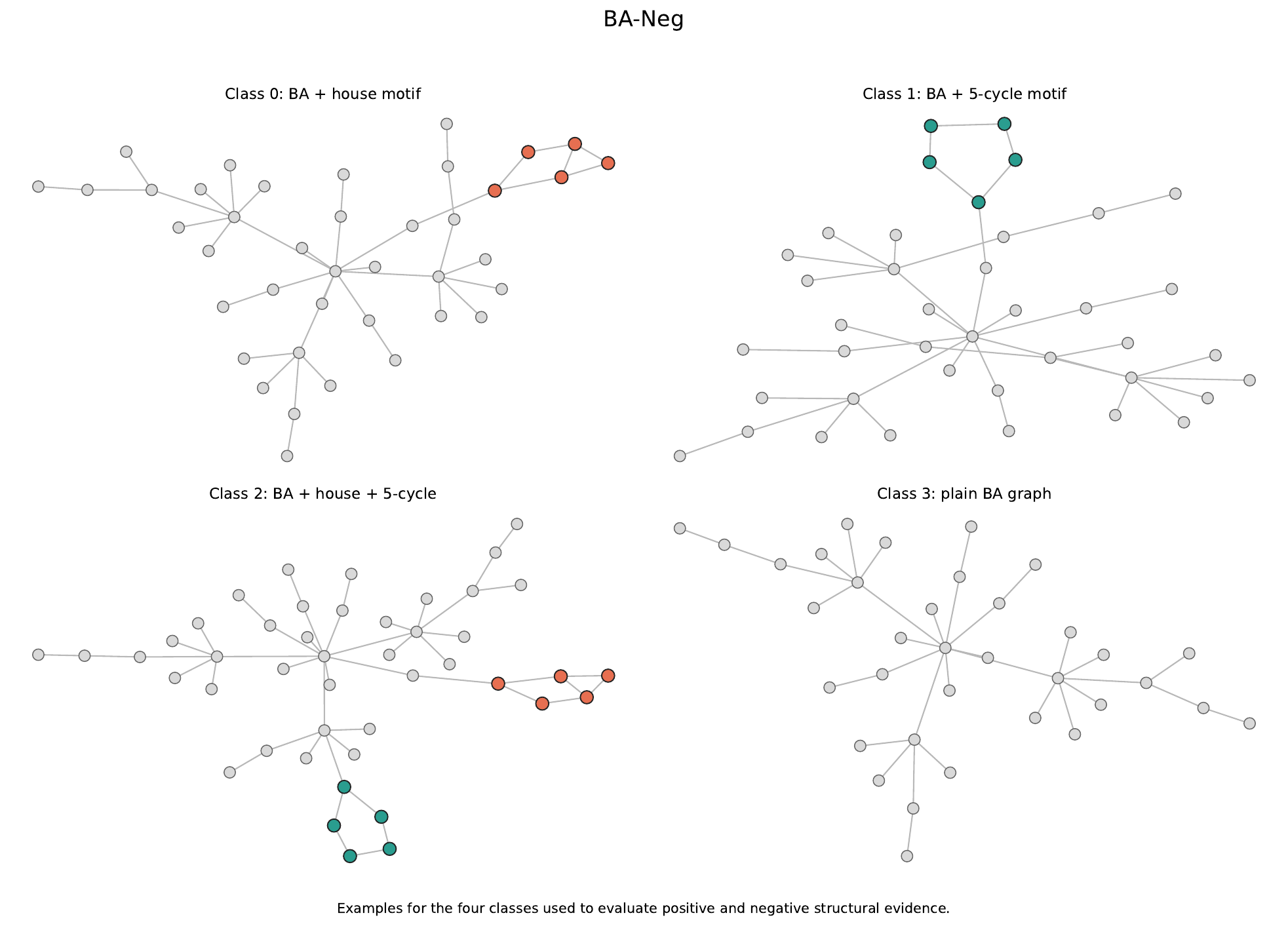}
  \caption{Example graphs of the BA-Neg dataset.}
  \label{fig:ba_neg_examples}
\end{figure*}

\paragraph{BA-Neg.}
As shown in \cref{fig:ba_neg_examples}, \textsc{BA-Neg} is a four-class synthetic benchmark designed to require both positive and negative structural evidence. The four classes are BA + house motif (Class 0), BA + 5-cycle motif (Class 1), BA + house + 5-cycle motifs (Class 2), and plain BA graphs without either motif (Class 3). This dataset is especially important for {\mymethod} because a faithful symbolic explanation must not only identify which motif is present, but also which motif is absent. It therefore tests the value of negative literals and the ability of the shared rule bank to encode absence-based reasoning.

\begin{figure*}[ht]
  \centering
  \includegraphics[width=\textwidth]{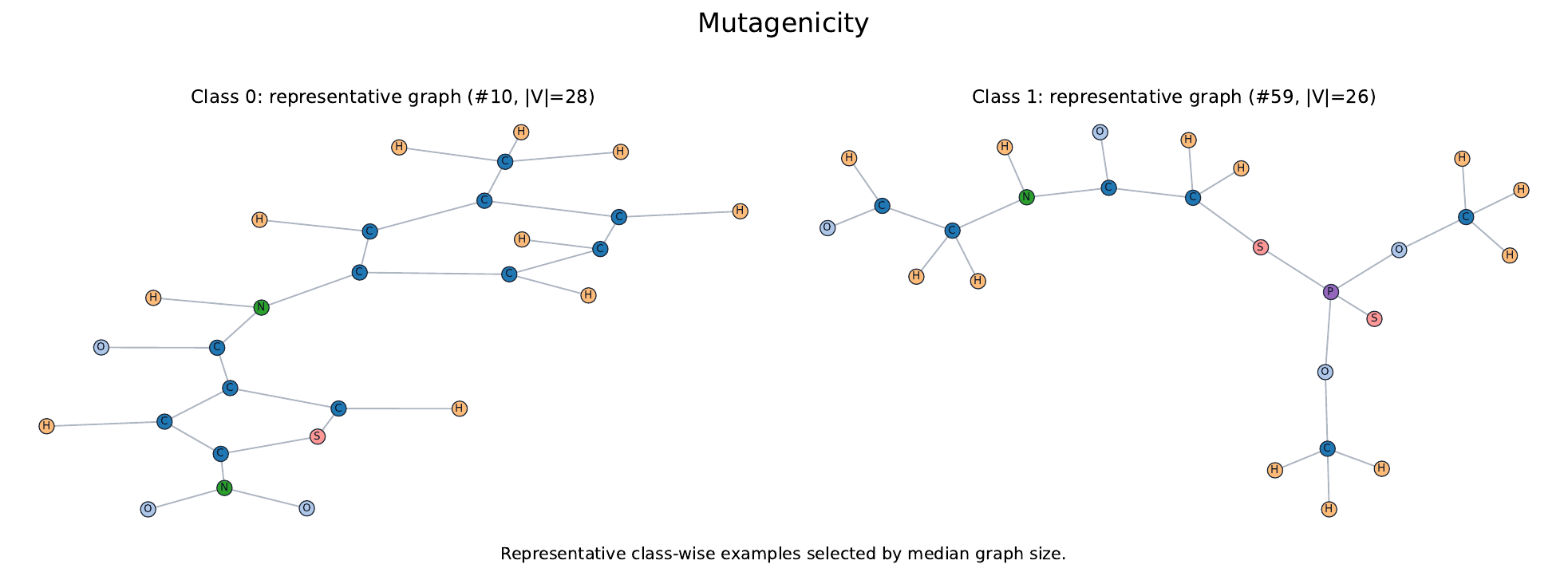}
  \caption{Example graphs of the Mutagenicity dataset.}
  \label{fig:Mutagenicity_examples}
\end{figure*}

\paragraph{Mutagenicity.}
As shown in \cref{fig:Mutagenicity_examples}, \textsc{Mutagenicity} is a real-world molecular graph classification dataset where the task is to predict whether a molecule is mutagenic. Compared with the synthetic datasets, the class signal is much noisier and the graph sizes vary substantially, which makes both prediction and explanation harder. {\mymethod} uses this dataset to evaluate whether the learned global rules remain faithful and interpretable when the underlying decision boundary is no longer defined by a single planted motif. Critical motifs in Class 0 include -NO2, -NH2, and carbon ring.

\begin{figure*}[ht]
  \centering
  \includegraphics[width=\textwidth]{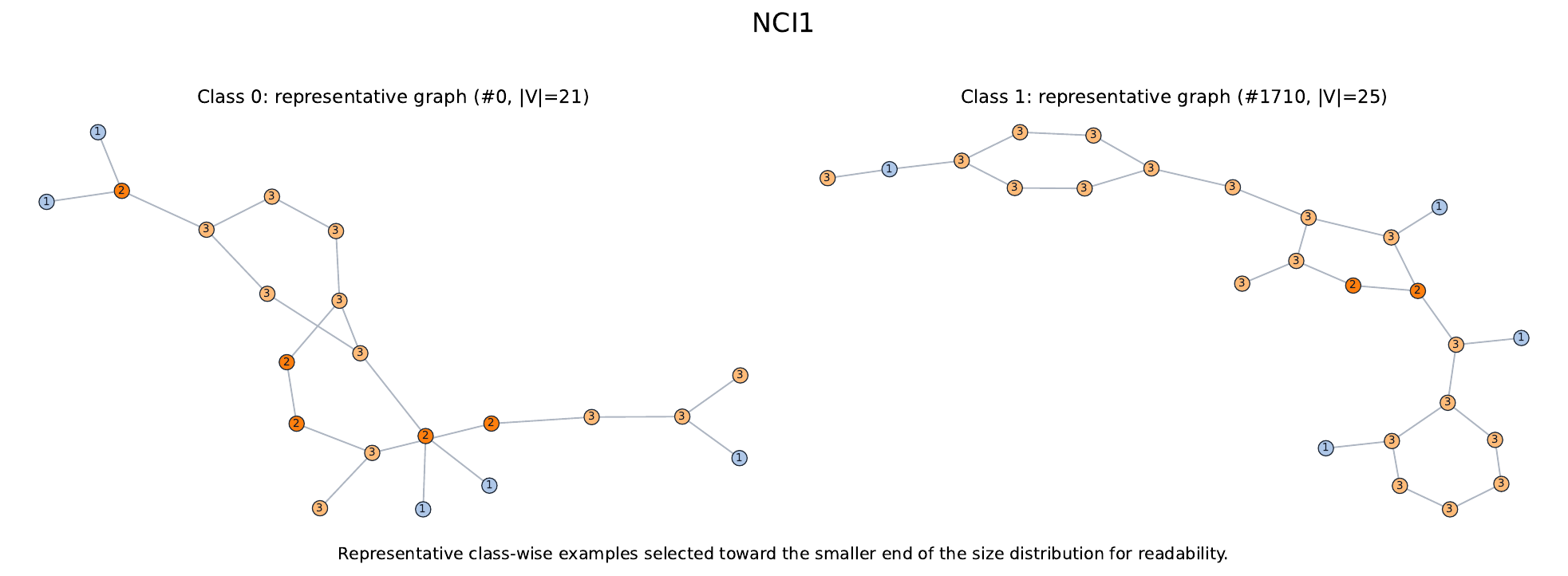}
  \caption{Example graphs of the NCI1 dataset.}
  \label{fig:NCI1_examples}
\end{figure*}

\paragraph{NCI1.}
As shown in \cref{fig:NCI1_examples}, \textsc{NCI1} is a real-world molecular graph classification benchmark derived from chemical compounds screened against non-small cell lung cancer. The graphs are structurally heterogeneous and the predictive cues are more distributed than in the synthetic benchmarks. 

\subsection{Implementation details.}
\paragraph{Compute resources.}
All experiments were run on CPU on a MacBook Pro with an Apple M4 Max chip, 16 CPU cores (12 performance cores and 4 efficiency cores), and 128GB unified memory. 
No external GPU or cloud compute was used. 
The runtime results reported in \cref{tab:agreement} were measured under this CPU-only setting and are reported in \(10^3\) seconds.

\paragraph{Base GNNs.}
Implementation details of the GNNs for explainablity in this paper is shown in \cref{tab:gnn_stats}. The GNNs were trained and evaluated by randomly splitting the datasets into training/validation/testing sets at 0.8/0.1/0.1 ratio. The random seed we used was 1234 while we split the data.

\begin{table*}[htp]
\caption{Details of the GIN backbones used in our experiments under seed 1234. All graph classification experiments use an 80/10/10 train/validation/test split. ``hidden'' denotes the latent dimension size, and $L$ is the number of GIN layers.}
\label{tab:gnn_stats}
\begin{center}
\scalebox{0.9}{
\begin{tabular}{l|ccccc}
\toprule
{\bf Datasets}
& {\bf BA-2Motifs}
& {\bf BAMultiShapes}
& {\bf BA-Neg}
& {\bf Mutagenicity}
& {\bf NCI1} \\
\midrule
number of GNN layers ($L$) & 3 & 3 & 3 & 3 & 3 \\
hidden                     & 32 & 32 & 32 & 64 & 64 \\
global pooling             & mean & mean & mean & mean & mean \\
layer type                 & GIN & GIN & GIN & GIN & GIN \\
node feature dim           & 10 & 10 & 10 & 14 & 37 \\
number of classes          & 2 & 2 & 4 & 2 & 2 \\
split ratio                & 80/10/10 & 80/10/10 & 80/10/10 & 80/10/10 & 80/10/10 \\
split seed                 & 1234 & 1234 & 1234 & 1234 & 1234 \\
\midrule
train acc                  & 1.00 & 1.00 & 1.00 & 0.882 & 0.935 \\
validation acc             & 1.00 & 1.00 & 1.00 & 0.873 & 0.869 \\
test acc                   & 1.00 & 1.00 & 1.00 & 0.791 & 0.830 \\
\bottomrule
\end{tabular}}
\end{center}
\end{table*}

For synthetic datasets, we use a 3-layer GIN with hidden size 32; for the real molecular datasets, we use a 3-layer GIN with hidden size 64.
The graph-level readout is mean pooling, followed by a two-layer MLP prediction head. 

\paragraph{Hyperparameter setting.}

We use GIN as the frozen base classifier for all datasets. Unless otherwise noted, we train the models for 1000 epochs, adopt \textsc{mean} as the local concept aggregation operator, enable negative literals, and allow signed rule weights. This choice keeps the rule space compact and improves interpretability without introducing an additional combinatorial expansion of candidate literals. The hyperparameters in \cref{eq:cnf_bce} are fixed where $\lambda_{\text{lit}}=0.001, \lambda_\text{clause}=0.001$ for mild structure-preserving regularizations, and $\gamma=0.1$. 
We use three random seeds for most datasets, namely \{1234, 4321, 123\}, and report the corresponding multi-run mean and standard deviation in \cref{tab:fidelity}, \ref{tab:agreement} and \ref{tab:rule_ablation}, as well as \cref{fig:rule_ablation}.

The remaining hyperparameters are selected based on dataset complexity. In general, we scale the number of local concepts $N_{lc}$, number of global concepts $N_{gc}$, and the size of the rule bank $N_r$ with the structural diversity of the dataset. Simpler synthetic datasets use smaller values to encourage compact rules, while more heterogeneous synthetic or real-world datasets use larger values to preserve fidelity. Likewise, number of literals per clause $N_\ell$ and number of clauses per rule $N_c$ control rule expressiveness: we use small clauses ($N_\ell=2$) throughout to keep rules readable, and allow more clauses only when the dataset requires more compositional structure. The rule-weight learning rate $\eta_w$, CNF learner learning rate $\eta_{\operatorname{CNF}}$, sparsity coefficient $\lambda_w$, and presence smoothing temperature $\tau$ are tuned per dataset to balance faithfulness, compactness, and optimization stability. 

Dataset-specific settings are as follows:

BA-2Motifs: $N_{lc}=3$, $N_{gc}=6$, $N_r=5$, $N_\ell=2$, $N_c=3$, $\lambda_w=0.01$, $\eta_{\operatorname{CNF}}=0.02$, $\eta_w=0.002$, $\tau=0.001$.
This dataset has a very clean planted decision boundary, so a small concept bank and a small rule bank are sufficient; stronger sparsity further encourages concise symbolic recovery.
 
BA-Neg: $N_{lc}=3$, $N_{gc}=10$, $N_r=10$, $N_\ell=2$, $N_c=3$, $\lambda_w=0.01$, $\eta_{\operatorname{CNF}}=0.02$, $\eta_w=0.2$, $\tau=0.001$.
This dataset requires both positive and negative evidence, so we use a larger rule bank than BA-2Motifs to capture multiple class-specific absence/presence combinations; the larger weight learning rate was used to stabilize multiclass logit mimicry under this more contrastive setting.

BAMultiShapes: $N_{lc}=3$, $N_{gc}=20$, $N_r=20$, $N_\ell=2$, $N_c=4$, $\lambda_w=0.001$, $\eta_{\operatorname{CNF}}=0.02$, $\eta_w=0.02$, $\tau=0.001$.
This dataset contains multiple injected motifs and combinatorial class structure, so we use a larger concept inventory and a more expressive CNF parameterization.

Mutagenicity: $N_{lc}=3$, $N_{gc}=30$, $N_r=20$, $N_\ell=2$, $N_c=3$, $\lambda_w=0.001$, $\eta_{\operatorname{CNF}}=0.02$, $\eta_w=0.02$, $\tau=0.001$.
The real-world molecular graphs are more heterogeneous than the synthetic datasets, so we enlarge the concept and rule spaces to preserve fidelity while keeping clause size small for readability.

NCI1: $N_{lc}=3$, $N_{gc}=20$, $N_r=20$, $N_\ell=2$, $N_c=3$, $\lambda_w=0.001$, $\eta_{\operatorname{CNF}}=0.005$, $\eta_w=0.01$, $\tau=0.001$.
NCI1 is structurally diverse but less cleanly separable, so we keep a moderately large concept/rule bank while using a smaller CNF learning rate and a sharper presence temperature to make rule induction more conservative.

\section{Additional Experimental Results} \label{app:more_results}

\subsection{Results of Evaluation over Delta Logit and TAPF} \label{app:delta_logit_tapf}
\begin{table}[htp]
    \caption{Top-$k$ rule ablation on the test set. For each dataset, we report the change in reconstructed logits after individually ablating the top-$k$ critical rules and averaging their effects (RQ2). TC is the Target-Class $\Delta$ logit, NTA is the Non-Target Average $\Delta$ logit.}
    \label{tab:rule_ablation}
    \centering
    \scalebox{0.7}{
    \begin{tabular}{lcccccccccc}
    \toprule
    \multirow{2.5}{*}{Top-$k$} & \multicolumn{2}{c}{BA-2Motifs} & \multicolumn{2}{c}{BAMultiShapes} & \multicolumn{2}{c}{BA-Neg} & \multicolumn{2}{c}{Mutagenicity} & \multicolumn{2}{c}{NCI1} \\
    \cmidrule(r){2-3}\cmidrule(r){4-5}\cmidrule(r){6-7}\cmidrule(r){8-9}\cmidrule(r){10-11}
     & TC $(\uparrow)$ & NTA $(\downarrow)$ & TC $(\uparrow)$ & NTA $(\downarrow)$ & TC $(\uparrow)$ & NTA $(\downarrow)$ & TC $(\uparrow)$ & NTA $(\downarrow)$ & TC $(\uparrow)$ & NTA $(\downarrow)$ \\
    \midrule
    Top-10 & {0.55}\footnotesize{$\pm$0.27} & {-0.72}\footnotesize{$\pm$0.14} & {0.97}\footnotesize{$\pm$0.39} & {-0.95}\footnotesize{$\pm$0.45} & {1.74}\footnotesize{$\pm$0.35} & {-0.93}\footnotesize{$\pm$0.16} & {0.82}\footnotesize{$\pm$0.23} & {-0.82}\footnotesize{$\pm$0.23} & {1.70}\footnotesize{$\pm$0.90} & {-1.77}\footnotesize{$\pm$1.06} \\
    Top-3 &  {2.70}\footnotesize{$\pm$1.11} & {-2.68}\footnotesize{$\pm$1.16} & {2.96}\footnotesize{$\pm$1.39} & {-2.88}\footnotesize{$\pm$1.66} & {3.60}\footnotesize{$\pm$0.46} & {-1.59}\footnotesize{$\pm$0.10} & {1.51}\footnotesize{$\pm$0.50} & {-1.52}\footnotesize{$\pm$0.44} & {3.94}\footnotesize{$\pm$2.83} & {-4.14}\footnotesize{$\pm$3.35} \\
    Top-1 &  {4.86}\footnotesize{$\pm$1.85} & {-4.79}\footnotesize{$\pm$1.59} & {4.95}\footnotesize{$\pm$2.71} & {-4.77}\footnotesize{$\pm$3.24} & {4.41}\footnotesize{$\pm$0.89} & {-2.13}\footnotesize{$\pm$0.10} & {1.91}\footnotesize{$\pm$0.70} & {-1.88}\footnotesize{$\pm$0.64} & {8.61}\footnotesize{$\pm$8.52} & {-9.62}\footnotesize{$\pm$10.12} \\
    \bottomrule
    \end{tabular}}
\end{table}

\paragraph{Results of Delta Logit.}
As mentioned in \cref{sec:exp:rule_impact}, we now present the raw results of $\Delta$ logit in \cref{tab:rule_ablation}. We note that RQ2 is not intended to measure the overall perturbation fidelity of {\mymethod} to the base GNN. The fidelity has already been evaluated in RQ1 (\cref{sec:exp:logit_fidelity}). 
Instead, it asks whether the top active rules selected for each test instance provide effective evidence for the base GNN's prediction. 
A useful rule-level explanation should not only be active on the graph; it should also push the reconstructed decision toward the GNN-predicted class. 

As shown in \cref{tab:rule_ablation}, the selected top rules consistently behave as prediction-supporting evidence. 
Across all five datasets and all values of \(k\), TC is positive and NTA is negative. 
This means that the critical active rules identified by {\mymethod} increase the logit of the GNN-predicted class while decreasing the average logit of non-target classes. 
Therefore, the selected rules actively shape the reconstructed decision in the same direction as the base GNN prediction. 

\begin{table}[htb]
    \centering
    \caption{APF and TAPF on BA-Neg.
BA-Neg is multiclass, so APF and TAPF are not theoretically equivalent.
For the binary datasets, TAPF equals APF and is therefore omitted. }
    \begin{tabular}{lccccc}
    \toprule
    BA-Neg & w/o rules & w/o merge & w/o $\neg$ & w/ DNF & {\mymethod} (Ours) \\
    \midrule
    APF & {50.9}\footnotesize{$\pm$11.5} & {84.0}\footnotesize{$\pm$2.7} & {71.6}\footnotesize{$\pm$4.1} & {91.6}\footnotesize{$\pm$7.2} & {92.3}\footnotesize{$\pm$6.9}\\
    TAPF  & {51.2}\footnotesize{$\pm$11.7} & {84.3}\footnotesize{$\pm$2.6} & {71.8}\footnotesize{$\pm$4.1} & {91.8}\footnotesize{$\pm$7.1} & {92.5}\footnotesize{$\pm$7.1}\\
    \bottomrule
    \end{tabular}
    \label{tab:tapf_ba_neg}
\end{table}

\paragraph{Results of TAPF.}
We use APF and TAPF to evaluate probability-level fidelity from two complementary perspectives. 
APF measures full-distribution agreement between the base GNN and {\mymethod} by comparing the entire softmax probability vector. 
It therefore captures whether {\mymethod} preserves not only the predicted class, but also the relative probabilities assigned to alternative classes. 
TAPF, in contrast, focuses only on the probability assigned to the GNN-predicted class. 
It directly measures whether {\mymethod} preserves the base GNN's confidence in its own prediction. 
Thus, APF evaluates distribution-level fidelity, while TAPF evaluates predicted-class confidence preservation. 

For binary classification, these two metrics are equivalent because the probability of one class fully determines the probability of the other. 
For multiclass tasks, however, they can differ: a method may preserve the GNN-predicted class confidence while distorting the probability mass assigned to non-target classes. 
BA-Neg is therefore the only dataset where we report both APF and TAPF.

As shown in \cref{tab:tapf_ba_neg}, TAPF closely follows APF across all variants on BA-Neg. 
This indicates that the probability-level fidelity of {\mymethod} is not achieved by preserving only the GNN-predicted class while distorting the remaining classes. 
Instead, target-class confidence preservation and full-distribution agreement move together.

The BA-Neg results also highlight the importance of negative evidence. 
Removing negation causes a substantial degradation, with APF dropping from \(92.3\%\) to \(71.6\%\) and TAPF dropping from \(92.5\%\) to \(71.8\%\). 
This is expected because BA-Neg contains classes defined by the absence of specific concepts. 
The result shows that the base GNN can learn absence-based evidence from dataset-level regularities, and that {\mymethod} needs negative literals to faithfully explain such behavior. 
In this sense, the \(w/o~\neg\) variant does not merely lose a syntactic option; it loses the ability to represent an important type of decision evidence used by the GNN.

\subsection{Global Concepts Extracted by TreeX} \label{app:gc_my}
\begin{figure}[htp]
    \centering
    \includegraphics[width=0.7\linewidth]{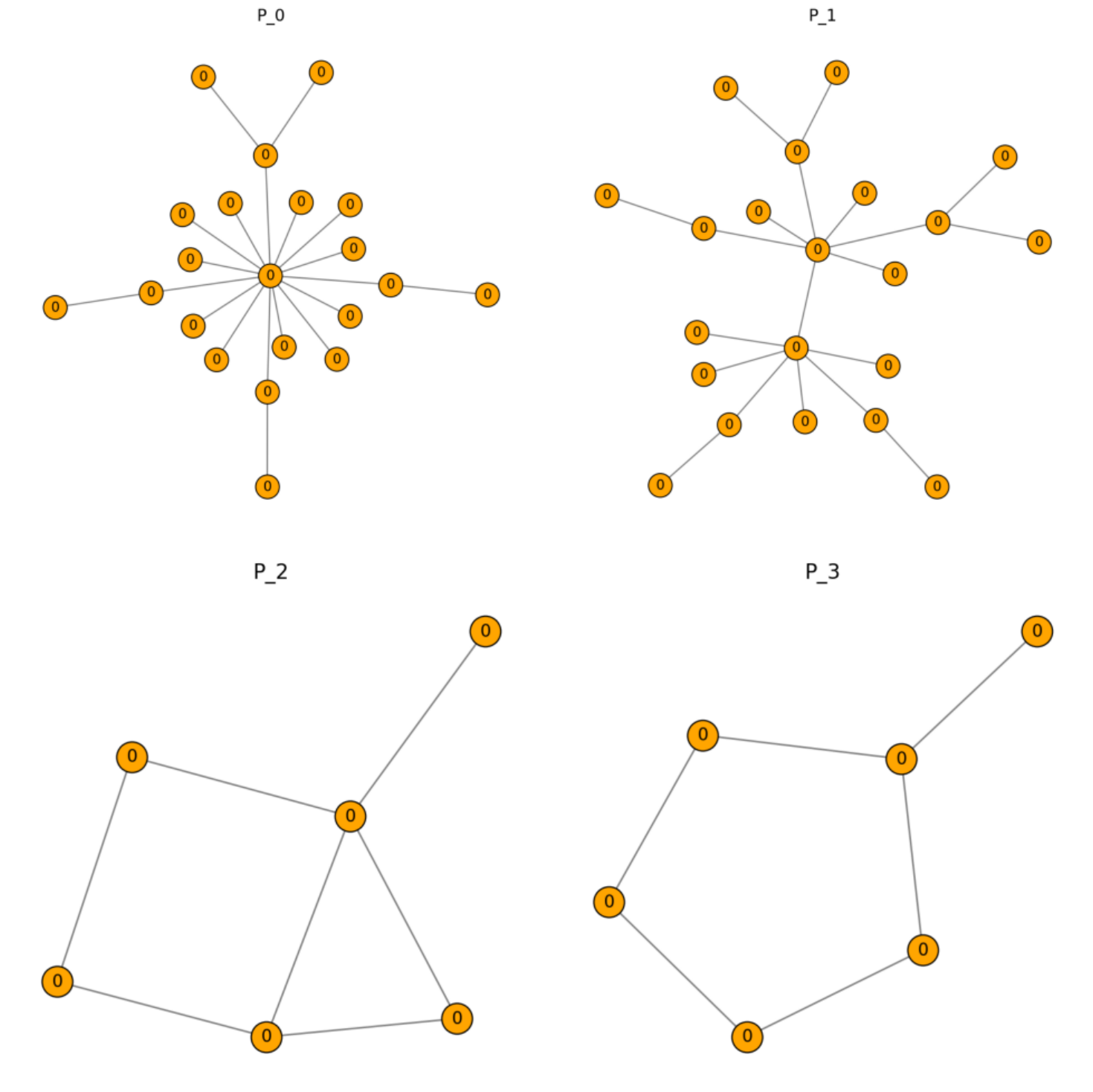}
    \vspace{1mm}
    \begin{small}
    \[
    \begin{array}{ll}
    \neg p_3 \land p_2 
    & \text{largest logit contribution: Class 1}\\
    \neg p_2 \land p_3
    & \text{largest logit contribution: Class 0}\\
    \end{array}
    \]
    \end{small}
    \caption{
    Global concepts extracted by {\mymethod} on BA-2Motifs, which serve as reusable literals for rule-level logit reconstruction.
    Example rules over these concepts are shown below the figure.
    }
    \label{fig:gc_ba_2motifs}
\end{figure}

\begin{figure}[htp]
    \centering
    \includegraphics[width=\linewidth]{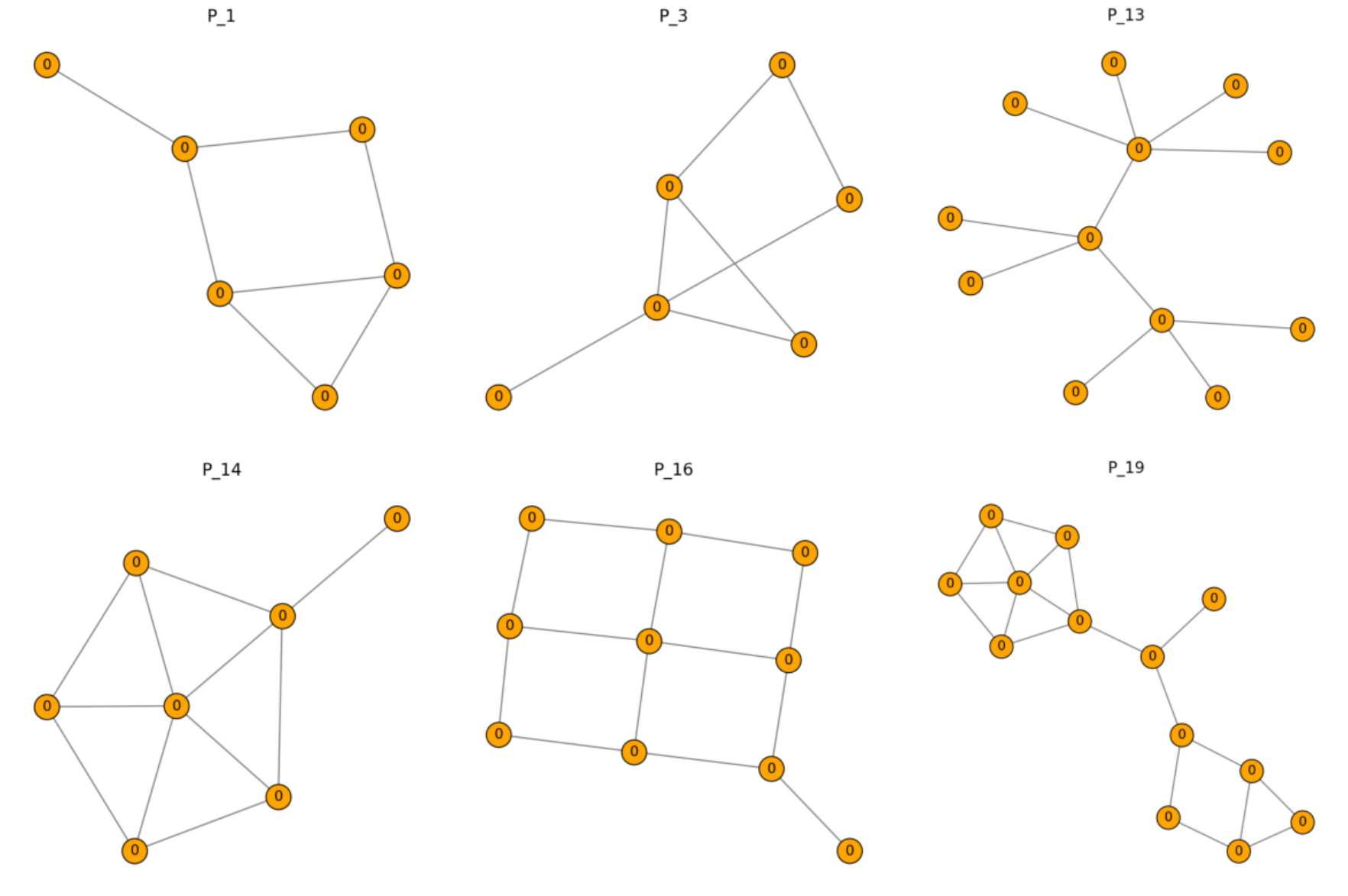}
    \vspace{1mm}
    \begin{small}
    \[
    \begin{array}{l|ll}
    r_0 & (p_{11}\lor \neg p_1) \land(p_{19}\lor\neg p_{14}) \land(p_{19}\lor\neg p_{16})\land(p_{16}\lor\neg p_{19})
    & \text{largest logit contribution: Class 0}\\
    r_1 & (p_{16}\lor \neg p_1)\land \neg p_3 \land (p_1 \lor \neg p_{16})\land p_{14}
    & \text{largest logit contribution: Class 0}\\
    r_2 & (p_7\lor\neg p_{14})\land (p_{1}\lor p_{14})\land p_{16} & \text{largest logit contribution: Class 1}\\
    r_3 & p_1\land (p_1 \lor p_3) \land (p_{14} \lor p_{16}) & \text{largest logit contribution: Class 1} \\
    r_4 & p_4\land p_{11}\land p_{12} & \text{largest logit contribution: Class 0} \\
    r_5 & p_3\land p_{19} & \text{largest logit contribution: Class 1}\\
    \end{array}
    \]
    \end{small}
    \caption{
    Global concepts and example rules extracted by {\mymethod} on BAMultiShapes.
    The upper panel shows representative global subgraph concepts, and the lower panel lists several examples of the learned rules over these concepts.
    The class shown on the right denotes the logit receiving the largest positive contribution from each rule. The concepts other than $p_1,p_3,p_{14},p_{16},p_{19}$ are various BA-style background graphs. 
    These rules recover the dataset-level logic of BAMultiShapes, where Class 0 contains plain BA graphs, single-motif graphs, and graphs with all three motifs, while Class 1 contains graphs with exactly two motifs.
    }
    \label{fig:gc_bamult}
\end{figure}

\begin{figure}[htp]
    \centering
    \includegraphics[width=\linewidth]{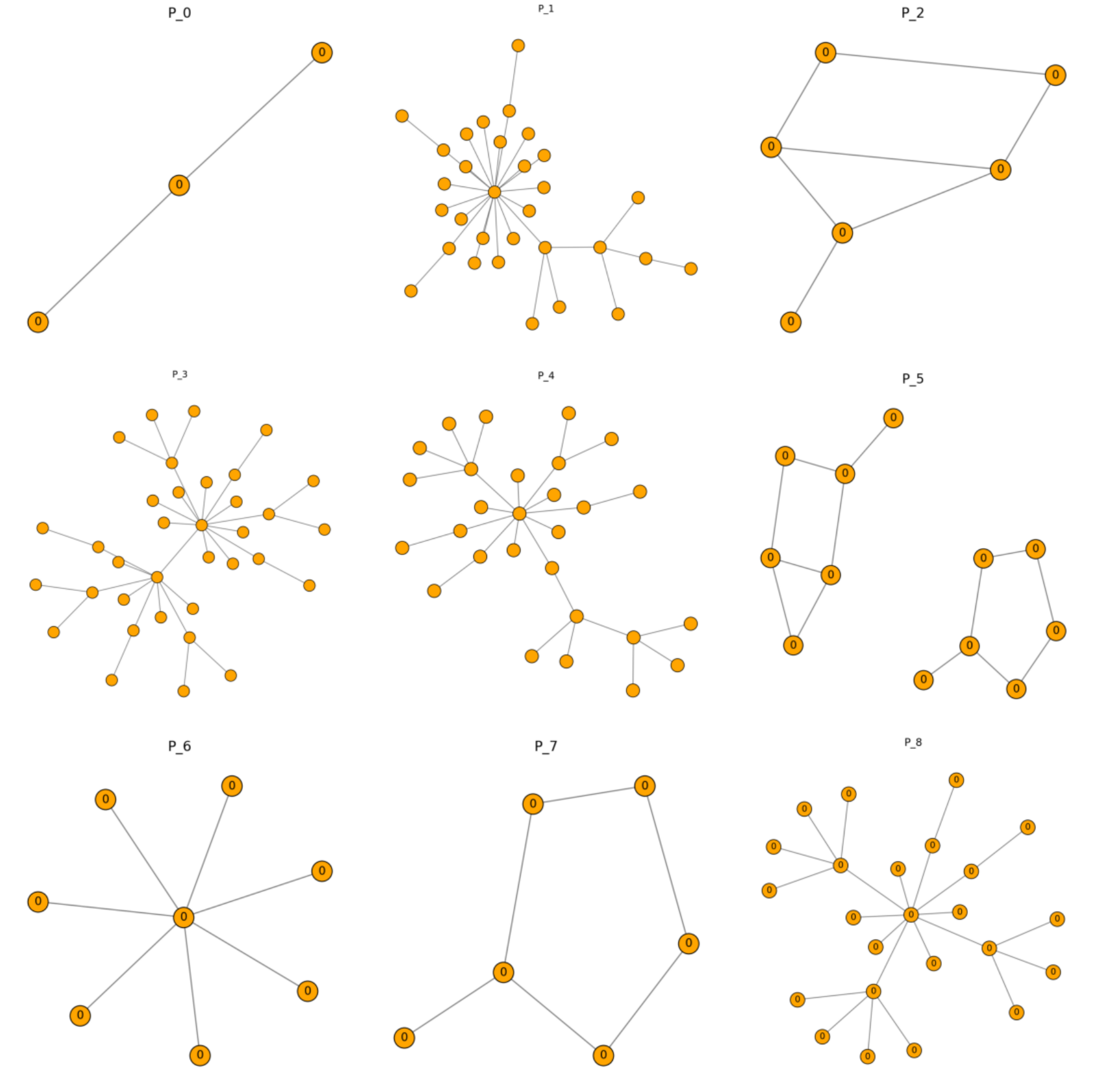}
    \vspace{1mm}
    \begin{small}
    \[
    \begin{array}{ll}
    (p_5\lor \neg p_3)\land(p_2\lor p_5)\land(p_5\lor p_7)
    & \text{largest  logit contribution: Class 2}\\
    \neg p_4\land(p_2\lor p_5)\land(p_5\lor p_7)
    & \text{largest logit contribution: Class 2}\\
    \neg(p_2\land p_5)
    & \text{largest logit contribution: Class 3}\\
    \neg p_7\land p_2
    & \text{largest logit contribution: Class 0}\\
    \neg p_2\land\neg p_5\land\neg p_7
    & \text{largest logit contribution: Class 3}\\
    \neg p_2\land p_7
    & \text{largest logit contribution: Class 1}
    \end{array}
    \]
    \end{small}
    \caption{
    Global concepts extracted by {\mymethod} on BA-Neg, which serve as reusable literals for rule-level logit reconstruction.
    Example rules over these concepts are shown below the figure.
    See \cref{fig:raw_rules_ba_neg_1_2,fig:raw_rules_ba_neg_3_5} for detailed rule logs.
    }
    \label{fig:gc_ba_neg}
\end{figure}

We present the global concepts and representative rule-based explanations learned by {\mymethod} in \cref{fig:gc_ba_2motifs,fig:gc_bamult,fig:gc_ba_neg,fig:gc_mutagenicity}. 
These visualizations show the global symbolic structure, but {\mymethod} provides richer explanation outputs beyond these figures. 
We include raw rule diagnostics and test-time rule grounding examples in \underline{\bf Appendix~\ref{app:rule_log} and \ref{app:test_ground}}. 

\paragraph{Rule interpretation on BAMultiShapes.}
\cref{fig:gc_bamult} shows representative global concepts and learned rules on BAMultiShapes. 
Since BAMultiShapes has a more compositional label structure than the other synthetic benchmarks, we analyze the rules learned by {\mymethod} in detail. 
This case study illustrates how {\mymethod} recovers rules that are faithful to the ground-truth motif logic, rather than merely identifying isolated class-associated patterns.

We focus our detailed rule analysis on BAMultiShapes because it provides a rare setting where both the ground-truth decision logic and the GNN behavior are verifiable. 
The dataset labels are generated from known motif combinations, and the base GNN achieves \(100\%\) accuracy, making it meaningful to ask whether the explainer recovers the true rule structure. 
This is less clear on molecular datasets, where the underlying label mechanism is not fully observable and an \(80\%\)-accurate GNN may rely on unknown or correlated chemical evidence. 
Thus, BAMultiShapes offers a cleaner testbed for demonstrating rule faithfulness to the ground-truth logic.

This dataset contains BA graphs with randomly inserted house, grid, and wheel motifs.
Class 0 includes plain BA graphs, graphs with a single motif, and graphs containing all three motifs, whereas Class 1 contains graphs with exactly two of the three motifs.
Therefore, a faithful model-level explanation should not merely detect individual motifs; it should recover the combinatorial rule that distinguishes one motif, two motifs, and three motifs.

The learned rules in \cref{fig:gc_bamult} align with this structure.
\begin{itemize}
    \item For instance, \(r_0\) can be satisfied by the conjunction \(p_{11}\land p_{19}\land p_{16}\), corresponding to BA graphs containing all three motif types, which belongs to Class 0. 
    Alternatively, $r_0$ can be satisfied by $(p_{11}\lor \neg p_1)\land \neg p_{14}\land\neg p_{16}\land \neg p_{19}$, corresponding to a plain BA graph, which also belongs to Class 0. 
    This rule $r_0$ represents two possible Class 0 style in a single CNF form. 
    \item Rule \(r_1\) can be satisfied either by \(p_{16}\land p_1\land \neg p_3\land p_{14}\) or by \(\neg p_1\land \neg p_3\land \neg p_{16}\land p_{14}\), capturing Class-0 cases such as all-three-motif graphs or single-wheel graphs, depending on which motif concepts are instantiated.
    \item Rule $r_2$ indicates either $p_7 \land (p_1 \lor p_{14})\land p_{16}$ or $\neg p_{14}\land p_1\land p_{16}$. It means either BA and a house/wheel and a grid, or exact 2 special motifs house and grid, both cases are valid for Class 1. 
    \item Rule \(r_3=p_1\land(p_1\lor p_3)\land(p_{14}\lor p_{16})\) also supports Class 1. Since \(p_1\) already satisfies the second clause, this rule effectively requires \(p_1\) together with either \(p_{14}\) or \(p_{16}\). It therefore represents a two-motif pattern: one motif concept is fixed by \(p_1\), while the second is provided by \(p_{14}\) or \(p_{16}\). This matches the Class-1 definition of BAMultiShapes.
    \item Rule \(r_4=p_4\land p_{11}\land p_{12}\) has its largest contribution to Class 0. The involved concepts correspond to BA-style background structures. 
    \item Rule \(r_5=p_3\land p_{19}\) contributes most strongly to Class 1. It is a compact two-concept rule: the simultaneous presence of \(p_3\) and \(p_{19}\) indicates a graph containing both a house and a wheel, matching the Class-1 condition. 
\end{itemize}

Overall, these examples show that {\mymethod} does not only identify isolated motifs.
It recovers rules that reflect the dataset's combinatorial label logic and assigns them class-wise logit contributions.
In particular, Class-1 rules correspond to exactly-two-motif evidence, while Class-0 rules capture either BA/single-motif cases or all-three-motif cases.
This shows that {\mymethod} can recover reusable model-level decision structure rather than merely producing rule-level artifacts. 

\begin{figure}[htp]
    \centering
    \includegraphics[width=\linewidth]{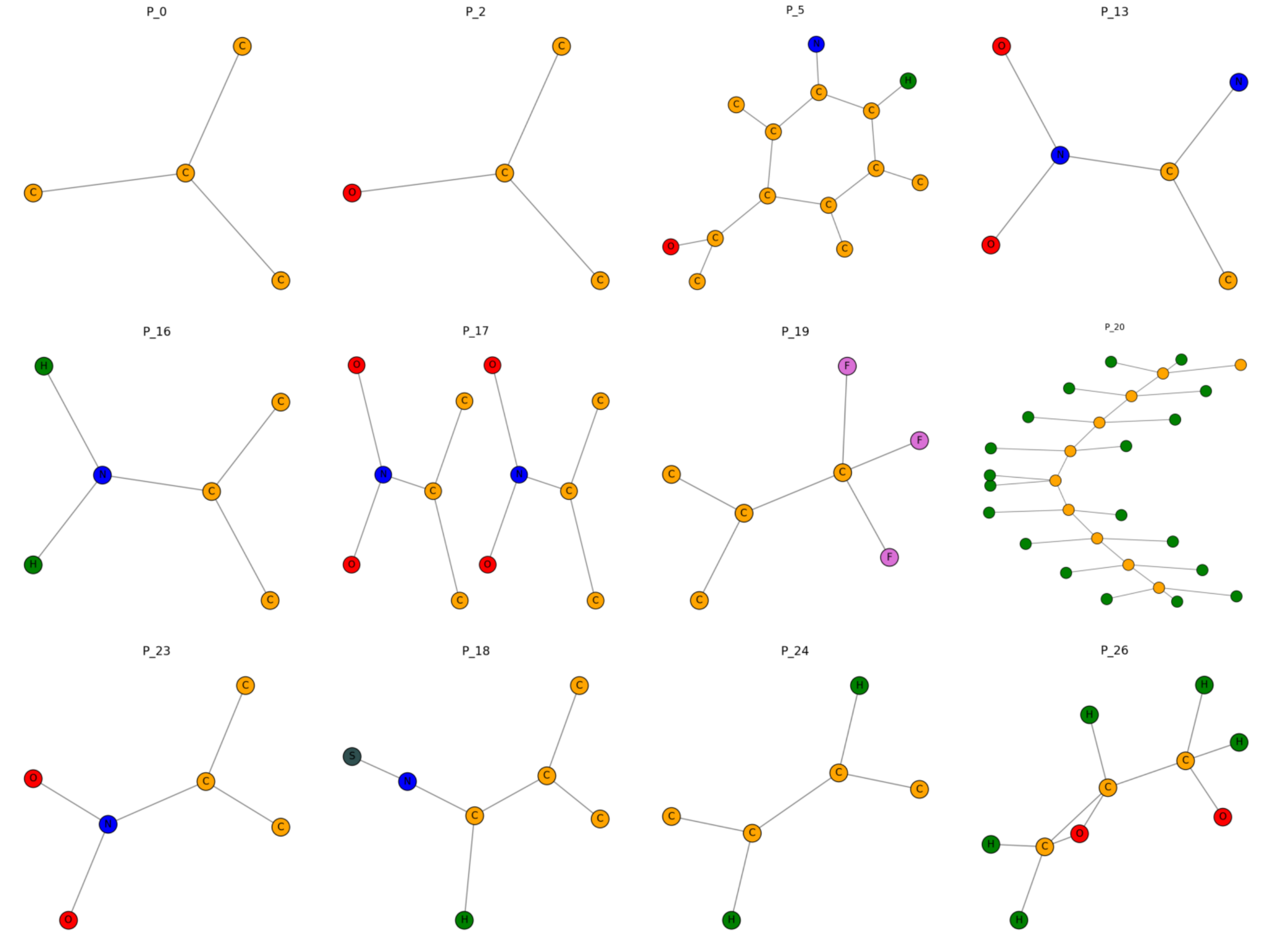}
    \vspace{1mm}
    \begin{small}
    \[
    \begin{array}{ll}
    \neg p_{16}\land (p_2\lor p_{19})\land (p_{19}\lor p_{20})
    & \text{largest  logit contribution: Class 1}\\
    p_5 \land (p_5\lor p_{13})\land (p_{13}\lor p_{17})
    & \text{largest  logit contribution: Class 0}\\
    p_0\land p_{13}\land p_{16} 
    & \text{largest  logit contribution: Class 0}\\
    (p_5\lor p_{26})\land (p_{18}\lor p_{24})
    & \text{largest  logit contribution: Class 0}\\
    \end{array}
    \]
    \end{small}
    \caption{
    Some global concepts and example rules extracted by {\mymethod} on Mutagenicity.
    The displayed concepts serve as reusable literals for rule-level logit reconstruction, and the listed rules show representative logit-contributing rules learned over these concepts.
    Several extracted concepts correspond to chemically meaningful functional groups known to be associated with mutagenic effects, such as nitro groups \((\mathrm{-NO}_2)\) in \(p_{13}\), \(p_{17}\), and \(p_{23}\), and amine-related groups \((\mathrm{-NH}_2)\) in \(p_{16}\).
    The examples also show that subtree-inspired concept mining can extract relatively large and structured molecular subgraphs ($p_{20}$), which are then reused in symbolic rules.
    }
    \label{fig:gc_mutagenicity}
\end{figure}

\begin{figure}[htp]
    \centering
    \includegraphics[width=\linewidth]{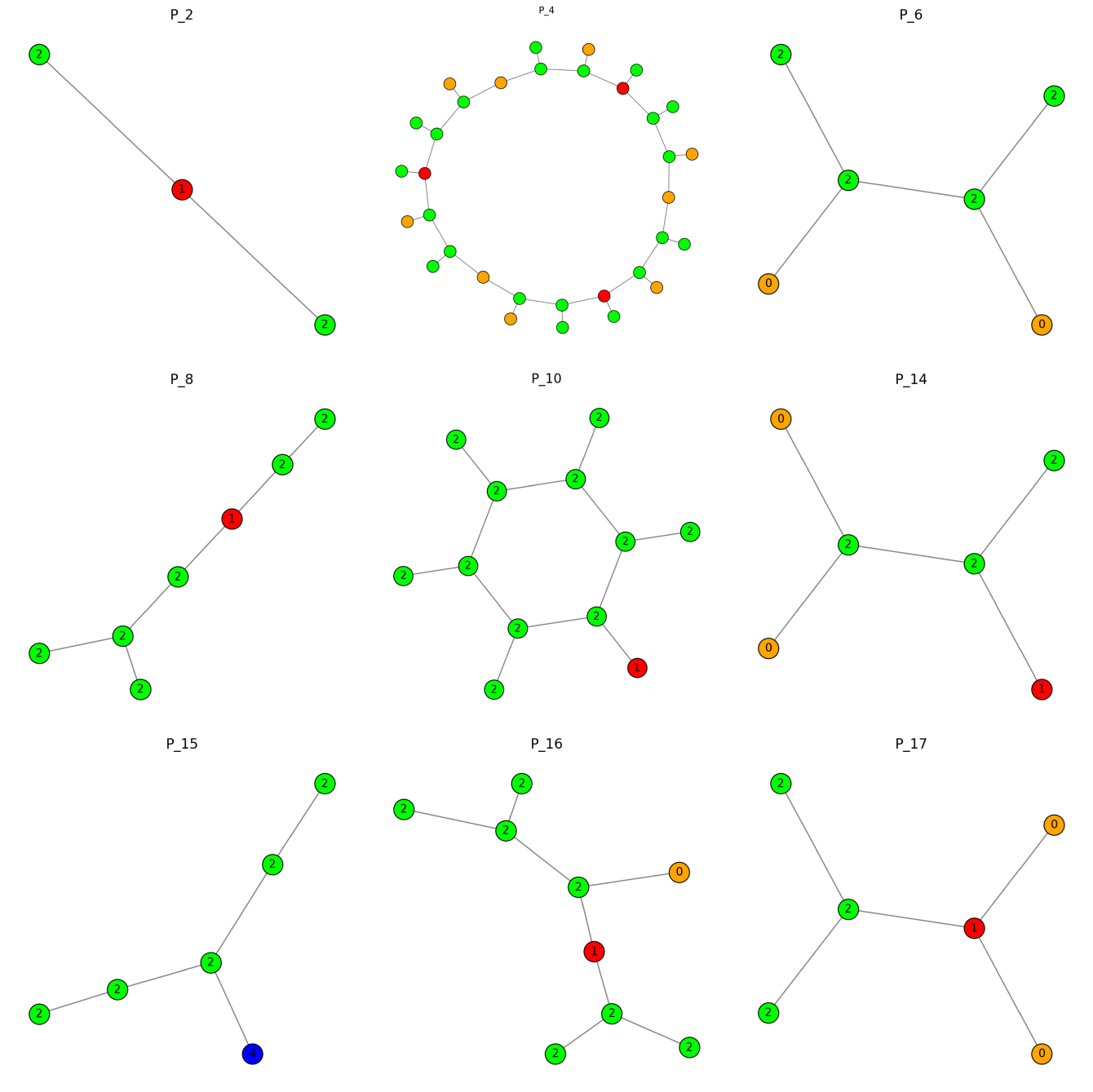}
    \vspace{1mm}
    \begin{small}
    \[
    \begin{array}{ll}
    (p_1\lor p_{10})\land p_8\land(p_{10} \lor p_{13})
    & \text{largest  logit contribution: Class 1}\\
    \neg p_{11}\land p_5 \land p_{10}
    & \text{largest  logit contribution: Class 1}\\
    p_9 \land p_{15}
    & \text{largest  logit contribution: Class 0}\\
    p_{16}\land p_{17}
    & \text{largest  logit contribution: Class 0}\\
    \end{array}
    \]
    \end{small}
    \caption{
    Some global concepts and example rules extracted by {\mymethod} on NCI1.
    The displayed concepts serve as reusable literals for rule-level logit reconstruction, and the listed rules show representative logit-contributing rules learned over these concepts.
    In NCI1, node labels encode discrete atom types and edges represent chemical bonds.
    Since the standard TU release does not provide an authoritative public lookup table from label IDs to human-readable atom symbols, we treat them as categorical atom-type identifiers. 
    Several extracted concepts correspond to structured molecular fragments, such as ring-like patterns in \(p_{10}\), which are reused by TreeX as rule literals.
    }
    \label{fig:gc_nci1}
\end{figure}

\subsection{Example Raw Rule Log by {\mymethod} on BA-Neg} \label{app:rule_log}
\cref{fig:raw_rules_ba_neg_1_2,fig:raw_rules_ba_neg_3_5} show representative raw rules learned by {\mymethod} on BA-Neg. 
Each rule is reported together with its symbolic expression, learned global weight, source class, training-set coverage and precision, activation statistics, and active-only ablation effect on the reconstructed logits. 
These logs illustrate that the learned rules are not merely post-hoc class labels attached to symbolic conditions. 
Instead, each rule has a measurable class-wise contribution to the rule-level readout.

For example, Rules 0 and 1 originate from Class 2 and have their largest positive active-only ablation effects on the Class-2 logit, while suppressing several competing classes. 
Rules 3, 4, and 5 similarly show class-specific logit effects: Rule 3 mainly supports Class 0, Rule 4 supports Class 3, and Rule 5 supports Class 1. 
This demonstrates that the source class discovered during class-wise rule generation is generally aligned with the downstream logit contribution after rules are merged into the shared rule bank.

The examples also highlight the importance of negative evidence. 
Several high-quality rules are defined partly or entirely through absent concepts, such as \(\neg p_7\land p_2\), \(\neg p_2\land\neg p_5\land\neg p_7\), and \(\neg p_2\land p_7\), which match the ground-truth rules of corresponding classes after mapping motifs to learned global concepts. 
These rules achieve high coverage and precision while producing strong class-wise logit effects, confirming that {\mymethod} can use both concept presence and concept absence as functional readout evidence.

\begin{figure}[htp]
\centering
\begin{rulelogbox}
### Rule 0

- Expression: `(P[5] OR NOT P[3]) AND (P[2] OR P[5]) AND (P[5] OR P[7])`
- Rule weight: 0.3613
- Source classes: [2]
- Coverage: 0.949
- Precision: 0.991
- Gap: 0.143
- F1: 0.969
- Threshold: 0.050
- Train weighted F1: 0.1617
- Train class discrimination power: 5.4358
- Train fire rate: 0.2333
- Train mean weighted activation: 0.0128
- Train active-only ablation delta per class: [-0.5946723222732544, -1.14484441280365, 2.745601177215576, -2.6901748180389404]
\end{rulelogbox}

\begin{rulelogbox}
### Rule 1

- Expression: `NOT P[4] AND (P[2] OR P[5]) AND (P[5] OR P[7])`
- Rule weight: 0.1549
- Source classes: [2]
- Coverage: 0.859
- Precision: 0.995
- Gap: 0.214
- F1: 0.922
- Threshold: 0.246
- Train weighted F1: 0.0615
- Train class discrimination power: 3.4622
- Train fire rate: 0.2104
- Train mean weighted activation: 0.0081
- Train active-only ablation delta per class: [-0.3549858629703522, -0.755026638507843, 1.7475519180297852, -1.7146141529083252]
\end{rulelogbox}
\begin{rulelogbox}
### Rule 2

- Expression: `P[2] AND P[5]`
- Rule weight: -0.1394
- Source classes: [0]
- Coverage: 0.992
- Precision: 0.610
- Gap: 0.011
- F1: 0.755
- Threshold: 0.017
- Train weighted F1: 0.0491
- Train class discrimination power: 0.3022
- Train fire rate: 0.4219
- Train mean activation: 0.0073
- Train mean weighted activation: -0.0010
- Train active-only ablation delta per class: [-0.012623555958271027, 0.10497571527957916, -0.15478305518627167, 0.14738908410072327]
\end{rulelogbox}
\caption{Example logical rules learned by {\mymethod} on BA-Neg, Rules 0--2.
Each box reports the learned rule expression, its global rule weight, source class, coverage and precision on the training split, activation statistics, and the active-only ablation effect on each class logit.
These raw logs show that {\mymethod} learns rules as weighted logit-contributing readout units rather than only class-wise symbolic summaries.}
\label{fig:raw_rules_ba_neg_1_2}
\end{figure}

\begin{figure}[htp]
\centering
\begin{rulelogbox}
### Rule 3

- Expression: `NOT P[7] AND P[2]`
- Rule weight: 0.0590
- Source classes: [0]
- Coverage: 0.988
- Precision: 0.972
- Gap: 0.973
- F1: 0.980
- Threshold: 1.000
- Train weighted F1: 0.0299
- Train class discrimination power: 6.2150
- Train fire rate: 0.2635
- Train mean activation: 0.2635
- Train mean weighted activation: 0.0156
- Train active-only ablation delta per class: [3.0899994373321533, -3.124983072280884, 1.2596594095230103, -1.8521641492843628]
\end{rulelogbox}
\begin{rulelogbox}
### Rule 4

- Expression: `NOT P[2] AND NOT P[5] AND NOT P[7]`
- Rule weight: 0.0627
- Source classes: [3]
- Coverage: 1.000
- Precision: 0.996
- Gap: 0.998
- F1: 0.998
- Threshold: 1.000
- Train weighted F1: 0.0296
- Train class discrimination power: 7.2977
- Train fire rate: 0.2438
- Train mean activation: 0.2438
- Train mean weighted activation: 0.0153
- Train active-only ablation delta per class: [-1.7749578952789307, -1.284544587135315, -3.730705499649048, 3.566969394683838]
\end{rulelogbox}
\begin{rulelogbox}
### Rule 5

- Expression: `NOT P[2] AND P[7]`
- Rule weight: 0.0548
- Source classes: [1]
- Coverage: 1.000
- Precision: 0.996
- Gap: 0.996
- F1: 0.998
- Threshold: 1.000
- Train weighted F1: 0.0267
- Train class discrimination power: 5.2629
- Train fire rate: 0.2552
- Train mean activation: 0.2552
- Train mean weighted activation: 0.0140
- Train active-only ablation delta per class: [-3.0420632362365723, 2.2208609580993652, -0.09059915691614151, -1.543188214302063]
\end{rulelogbox}
\caption{
Example logical rules learned by {\mymethod} on BA-Neg, Rules 3--5.
Each box reports the learned rule expression, its global rule weight, source class, coverage and precision on the training split, activation statistics, and the active-only ablation effect on each class logit.
These raw logs show that {\mymethod} learns rules as weighted logit-contributing readout units rather than only class-wise symbolic summaries.
}
\label{fig:raw_rules_ba_neg_3_5}
\end{figure}

\subsection{Test-Time Grounding Rule Visualizations of TreeX}
\label{app:test_ground}

\cref{fig:vis_bamult,fig:vis_BA_neg,fig:vis_mutagenicity} visualize test-time grounded explanations generated by {\mymethod} on various datasets. 
For each unseen graph, {\mymethod} selects the top active rules, grounds the satisfied literals to concrete global subgraph concepts, and reports the class-wise logit effect of ablating each rule. 
The rule-readout softmax closely matches the base GNN softmax in both examples, indicating that the grounded rules are not merely visually plausible explanations, but are sufficient to reconstruct the GNN's prediction behavior.

These examples highlight a key advantage of {\mymethod} over prior model-level GNN explainers. 
Prototype-based methods~\cite{yuan2020xgnn, wang2022gnninterpreter, yu2025mage} can show representative class patterns, and class-wise rule explainers~\cite{xuanyuan2023global,geng2026logicxgnn,azzolin2022global} can output formulas associated with a class, but they typically do not show how each rule changes the multiclass output of a specific test graph. 
In contrast, {\mymethod} instantiates global rules at test time and assigns each active rule a measurable contribution to the logits. 
Thus, an explanation can state not only which motif-level concepts are present or absent, but also whether a rule supports the predicted class, suppresses the competing class, and how strongly it contributes to the reconstructed decision.

\begin{figure*}[htp]
  \centering
  \includegraphics[width=\textwidth]{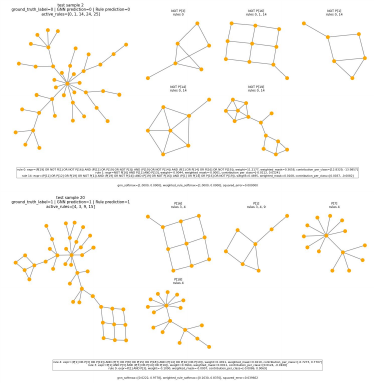}
  \caption{Test-time grounded explanations produced by {\mymethod} on BAMultiShapes.
    The upper subfigure shows a Class-0 graph, while the lower shows a Class-1 graph.
    For each graph, {\mymethod} grounds the top active rules to concrete subgraph concepts and reports each rule's class-wise logit effect computed by rule ablation.
    The base GNN softmax and the rule-readout softmax are shown below each subfigure, demonstrating that the grounded active rules closely reconstruct the GNN prediction.
    These examples show how {\mymethod} explains BAMultiShapes' combinatorial motif logic through functional rule-level evidence rather than through a single class-wise rule or prototype.}
  \label{fig:vis_bamult}
\end{figure*}

\paragraph{Test-time grounding on BAMultiShapes.}
\cref{fig:vis_bamult} visualizes test-time grounded explanations generated by {\mymethod} on BAMultiShapes. 
The upper panel shows a Class-0 example and the lower panel shows a Class-1 example. 
In BAMultiShapes, the labels are defined by motif combinations rather than by a single motif. 
The Class-0 and Class-1 examples demonstrate that {\mymethod} can reuse the same global concept bank to explain different combinatorial configurations, while still producing instance-specific rule activations and logit contributions. 
Therefore, {\mymethod} provides a more informative explanation object than a post-hoc class-wise summary: it links global symbolic rules, local grounded subgraphs, and quantitative logit-level effects within one test-time explanation.

\begin{figure*}[htp]
    \centering
    \includegraphics[width=\textwidth]{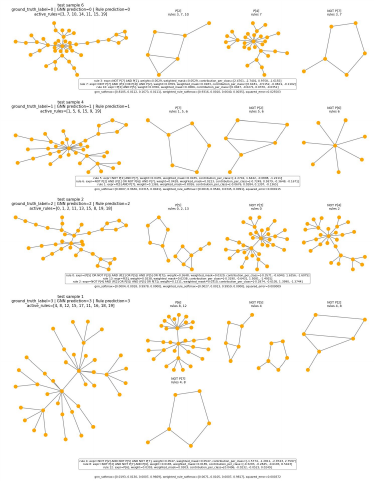}
    \caption{
    Test-time grounded explanations produced by {\mymethod} on BA-Neg.
    The four panels, from top to bottom, correspond to Class 0, Class 1, Class 2, and Class 3.
    For each unseen graph, {\mymethod} grounds the top active rules to concrete subgraph concepts, represents negative literals as required concept absence, and reports each rule's class-wise logit effect computed by ablation.
    The base GNN softmax and rule-readout softmax are shown below each panel, indicating that the grounded rules closely reconstruct the GNN prediction.
    }
    \label{fig:vis_BA_neg}
\end{figure*}

\paragraph{Test-time grounding on BA-Neg.}
We further visualize test-time grounded explanations on BA-Neg in \cref{fig:vis_BA_neg}. 
Unlike BAMultiShapes, whose main challenge is combinatorial motif composition, BA-Neg explicitly requires absence-based evidence: different classes are defined not only by which motifs are present, but also by which motifs are missing. 
From top to bottom, the examples cover all four classes. 
TreeX grounds positive literals to concrete subgraph concepts and represents negative literals as the absence of learned global concepts. 
The reconstructed rule-readout softmax closely matches the base GNN softmax, showing that the instantiated rules are sufficient to reproduce the GNN prediction. 
These examples demonstrate that TreeX can explain multiclass GNN decisions using both motif presence and motif absence, rather than only highlighting class-associated positive patterns.

\begin{figure*}[htp]
    \centering
    \includegraphics[width=\textwidth]{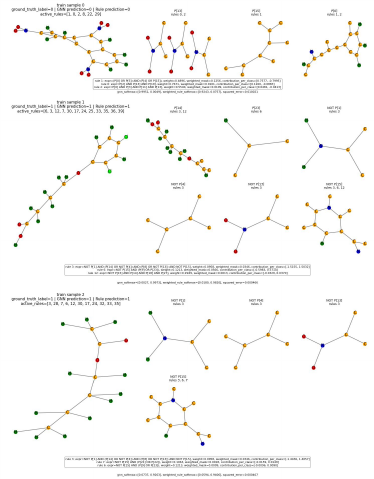}
    \caption{
    Test-time grounded explanations produced by {\mymethod} on Mutagenicity.
    From top to bottom, the panels show one Class-0 molecule and two Class-1 molecules.
    For each test molecule, {\mymethod} instantiates the top active rules, grounds satisfied positive literals to molecular subgraph concepts, visualizes negative literals as absent global concepts, and reports each rule's class-wise logit effect computed by ablation.
    The base GNN softmax, rule-readout softmax, and reconstruction error are shown below each panel.
    The examples include chemically meaningful concepts such as nitro groups \((\mathrm{-NO_2})\), amine-related groups \((\mathrm{-NH_2})\), and ring-like molecular substructures.
    }
    \label{fig:vis_mutagenicity}
\end{figure*}

\paragraph{Test-time grounding on Mutagenicity.}
\cref{fig:vis_mutagenicity} shows that {\mymethod} can provide chemically meaningful test-time rule grounding on real molecular graphs. 
In the Class-0 example, the molecule contains multiple nitro groups \((\mathrm{-NO_2})\), which are related to mutagenic effects. 
TreeX grounds the top active rules to these molecular substructures, and the corresponding ablation effects strongly push the reconstructed logits toward Class 0. 
The rule-readout softmax is also close to the base GNN softmax, indicating that these grounded rules capture the main evidence used by the GNN for this prediction.

The two Class-1 examples illustrate the complementary case. 
They do not contain obvious \(\mathrm{-NO_2}\) or \(\mathrm{-NH_2}\) groups, nor do they contain the larger connected aromatic-ring-like concept represented by \(p_4\) and $p_{15}$. 
TreeX therefore activates rules involving the absence of these concepts, producing logit effects that favor Class 1. 
This demonstrates that {\mymethod} does not only highlight present toxic substructures; it can also explain predictions through the absence of discriminative molecular concepts.

Some active rules contain literals that are only weakly instantiated or cannot be cleanly grounded to a salient local subgraph in the current molecule. 
In these cases, the corresponding activation mask is small, and the rule contributes only mildly to the final rule-readout logits. 
This behavior is desirable: rule-level explanations are weighted by their actual test-time activation, so visually weak or partially matched rules do not dominate the reconstructed prediction.

\subsection{Qualitative Comparison with Related Works} \label{app:gc_baselines}

\cref{sec:exp:compare_baseline} reports quantitative comparisons with class-wise rule-based explainers using Prediction Agreement (PA), the shared metric on which the baselines are designed to perform well. 
We now examine the explanations themselves. 
For rule- and concept-based methods, we compare with GLGExplainer~\cite{azzolin2022global}, LogicXGNN~\cite{geng2026logicxgnn}, and GCNeuron~\cite{xuanyuan2023global}. GraphTrail~\cite{armgaan2024graphtrail} is omitted from visualization because of its high concept-mining cost. 
For generation-based explainers such as XGNN~\cite{yuan2020xgnn}, GNNInterpreter\cite{wang2022gnninterpreter}, PAGE~\cite{shin2024page}, Gen-GraphEx~\cite{2025gengraphex}, and MAGE~\cite{yu2025mage}, the output is class-wise prototype graphs rather than executable rules, rule weights, test-time grounding, or logit contributions. 
Therefore, they are not directly comparable under our rule-level logit reconstruction metrics or prediction agreement metric, but we include examples below to show how their explanations differ from those produced by {\mymethod}. 
GnnXemplar~\cite{armgaan2026gnnxemplar} studies node classification by selecting exemplar nodes and using LLMs to derive natural-language signatures. This direction is complementary to our graph-classification setting, where the goal is to reconstruct graph-level GNN logits from grounded symbolic rules. We do not compare with GnnXemplar because it focuses on node classification, whereas we focus on graph classification.

\begin{figure}[htp]
    \centering
    \includegraphics[width=\textwidth]{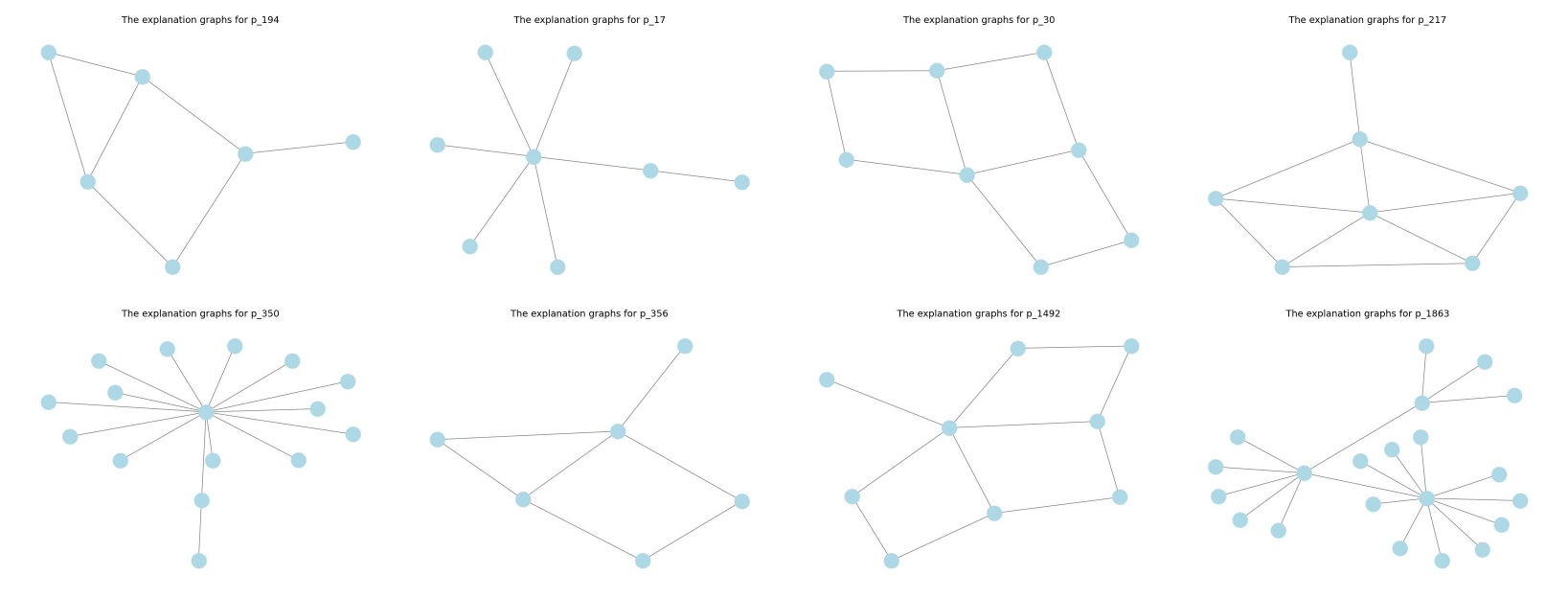}
    \vspace{1mm}
    \begin{small}
    \[
    \begin{array}{l}
    (\neg p_{30} \land \neg p_{217} \land \neg p_{1492} \land \neg p_{356} \land \neg p_{194})\lor (\neg p_{30} \land \neg p_{217} \land \neg p_{1492} \land p_{356} \land \neg p_{17}) \\
    \lor (\neg p_{30} \land \neg p_{217} \land p_{1492} \land p_{350})\lor (\neg p_{30} \land p_{217} \land p_{13} \land \neg p_{564} \land p_{1863})\lor (\neg p_{30} \land p_{217} \land p_{13} \land p_{564})\\ \lor (p_{30} \land \neg p_{24} \land \neg p_{564} \land \neg p_{18} \land p_{505})\lor (p_{30} \land \neg p_{24} \land \neg p_{564} \land p_{18} \land \neg p_{16}) \\ 
    \lor (p_{30} \land \neg p_{24} \land p_{564})\lor (p_{30} \land p_{24}) \Rightarrow{Class\ 0}\\
     \\
    (\neg p_{30} \land \neg p_{217} \land \neg p_{1492} \land \neg p_{356} \land p_{194})\lor (\neg p_{30} \land \neg p_{217} \land \neg p_{1492} \land p_{356} \land p_{17})\\
    \lor(\neg p_{30} \land \neg p_{217} \land p_{1492} \land \neg p_{350})\lor (\neg p_{30} \land p_{217} \land \neg p_{13} \land \neg p_{1492} \land \neg p_{191})\\
    \lor (\neg p_{30} \land p_{217} \land \neg p_{13} \land \neg p_{1492} \land p_{191})\lor (\neg p_{30} \land p_{217} \land \neg p_{13} \land p_{1492})\\
    \lor (\neg p_{30} \land p_{217} \land p_{13} \land \neg p_{564} \land \neg p_{1863})\lor (p_{30} \land \neg p_{24} \land \neg p_{564} \land \neg p_{18} \land \neg p_{505})\\
    \lor (p_{30} \land \neg p_{24} \land \neg p_{564} \land p_{18} \land p_{16}) \Rightarrow{Class\ 1}
    \end{array}
    \]
    \end{small}
    \caption{
    LogicXGNN~\cite{geng2026logicxgnn} explanation on BAMultiShapes.
    Their explanation is a pair of class-wise DNF rules over discrete predicates, with several predicates visualized above.
    Some predicates provide incomplete motif grounding, such as the grid-like predicates \(p_{30}\) and \(p_{1492}\).
    More importantly, parts of the Class-1 rule conflict with the known BAMultiShapes label logic.
    For instance,
    \((\neg p_{30}\land \neg p_{217}\land \neg p_{1492}\land \neg p_{356}\land p_{194})\)
    describes a graph with no grid-like motif, no wheel-like motif, and a house-like motif, which is a single-motif Class-0 case rather than an exactly-two-motif Class-1 case.
    Thus, although the learned rules are grounded, they do not clearly recover the combinatorial distinction between Class 0 and Class 1.
    }
    \label{fig:logicxgnn_bamult}
\end{figure}
\begin{figure}[htp]
    \centering
    \includegraphics[width=\textwidth]{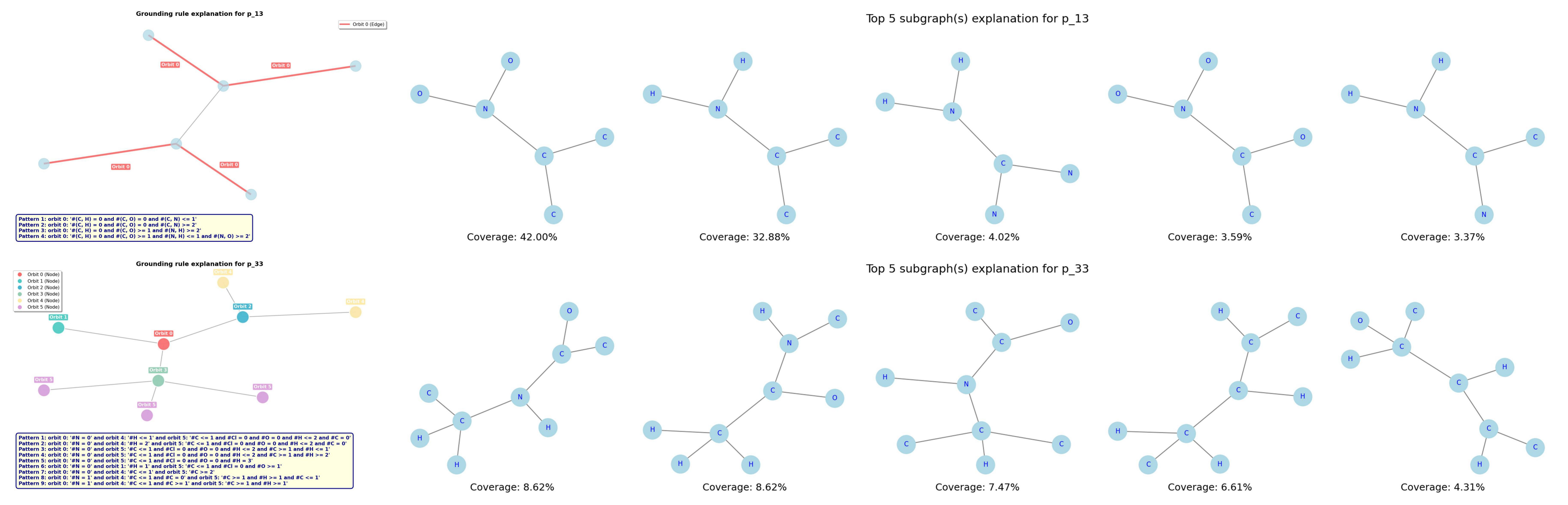}
    \vspace{1mm}
    \caption{
    Examples of LogicXGNN~\cite{geng2026logicxgnn} predicates on Mutagenicity.
    For this real world dataset, LogicXGNN grounds each predicate through topology-orbit conditions and shows multiple matched subgraph exemplars for the same predicate.
    For example, predicate \(p_{13}\) groups chemically different functional groups such as nitro-like \((\mathrm{-NO_2})\) and amine-like \((\mathrm{-NH_2})\) structures because they share similar local topology, while predicate \(p_{33}\) is defined by a complex set of orbit-level atom-count conditions.
    Thus, each predicate requires additional human interpretation to infer its chemical meaning, and the resulting predicate space does not necessarily align with chemically meaningful subgraph concepts.
    }
    \label{fig:logicxgnn_mutagenicity}
\end{figure}

\paragraph{Comparison with LogicXGNN.}
BAMultiShapes has known combinatorial label logic: Class 0 contains plain BA graphs, single motifs, and all three motifs, whereas Class 1 contains exactly two motifs. 
As shown in \cref{fig:logicxgnn_bamult}, LogicXGNN~\cite{geng2026logicxgnn} produces long class-wise DNF rules whose predicates are grounded but not always aligned with this logic. 
Some predicates only partially recover the intended motifs, and most Class-1 conditions correspond to Class-0 configurations. 
For instance, 
\((\neg p_{30}\land \neg p_{217}\land \neg p_{1492}\land \neg p_{356}\land p_{194})\)
describes a house-like motif without grid-like or wheel-like predicates, which matches Class 0 rather than the exactly-two-motif definition of Class 1. 
By contrast, as we discussed in Appendix~\ref{app:gc_my}, {\mymethod} recovers rules whose largest logit contributions align with the ground-truth motif combinations: Class-1 rules correspond to two-motif evidence, while Class-0 rules capture plain BA, single-motif, or all-three-motif cases. 
This comparison highlights the advantage of rule-based logit reconstruction: the rules are not only grounded, but also tied to class-wise logit effects that make the recovered decision logic easier to verify. 

We also examine the predicates extracted by LogicXGNN on Mutagenicity. 
Compared with BAMultiShapes, the predicate construction and grounding procedure on Mutagenicity becomes substantially more dataset-specific: predicates are described through enumeration-based topology-orbit conditions, and each predicate is associated with multiple matched subgraph exemplars rather than a single clean subgraph concept. 
This increases the burden on human users, who must inspect several exemplars and infer what chemical structure the predicate is intended to represent.

As shown in \cref{fig:logicxgnn_mutagenicity}, this issue is visible in predicates such as \(p_{13}\) and \(p_{33}\). 
Predicate \(p_{33}\) is specified by complex orbit-level atom-count conditions, making its semantic meaning difficult to read directly from the symbolic predicate. 
Predicate \(p_{13}\) further illustrates a chemical mismatch: nitro-like \((\mathrm{-NO_2})\) and amine-like \((\mathrm{-NH_2})\) structures are grouped into the same predicate because they share similar local topology. 
This is less aligned with chemical semantics, where such functional groups are usually interpreted separately.
We also do not observe a clean predicate corresponding to ring-like or benzene-like structures, although such structures are common and chemically meaningful in mutagenicity analysis.

From a design perspective, LogicXGNN constructs its predicates by matching a pre-enumerated local topology catalog. 
This may restrict the symbolic evidence space to patterns covered by the catalog, especially when the base GNN relies on structurally novel or mixed evidence at test time. 
This restriction may help explain why some learned rules are not fully aligned with the ground-truth logic on certain benchmarks.

In contrast, as discussed in Appendix~\ref{app:gc_my}, {\mymethod} mines global concepts as concrete subgraph prototypes from the GNN hidden space. 
On Mutagenicity, it separates \(\mathrm{-NO_2}\)-related and \(\mathrm{-NH_2}\)-related concepts and also extracts larger structured molecular subgraphs. 
More importantly, these concepts are not only displayed as visual prototypes: they serve as reusable literals in logical rules, are grounded on test molecules, and contribute quantitatively to reconstructed logits. 
Thus, {\mymethod} reduces the semantic gap between symbolic predicates and human-interpretable molecular evidence while preserving the ability to analyze class-wise logit effects. 

\paragraph{Comparison with GLGExplainer.}
GLGExplainer learns class-wise Boolean formulas over prototypes, but the prototypes are shown through multiple exemplar subgraphs rather than a single explicit grounded concept. 
As a result, the user must infer the meaning of each \(p_i\) by comparing several representatives. 

As shown in \cref{fig:glg_bamult}, on BAMultiShapes, this ambiguity makes the learned formulas difficult to verify against the known label logic. 
The Class-0 rule \(p_3\lor p_2\lor p_5\) does not explicitly distinguish plain BA, single-motif, and all-three-motif cases, while the long Class-1 disjunction is hard to map cleanly to the exactly-two-motif condition. 
In contrast, {\mymethod} grounds concepts as concrete subgraph prototypes and uses them in weighted rules with logit-level contributions, making the recovered combinatorial logic more directly inspectable.

We also examined GLGExplainer on Mutagenicity under the same setting as in their original paper, where only two prototypes are used.
As shown in~\cref{fig:glg_mutag}, unlike on synthetic datasets, where some prototypes may still loosely resemble motifs, the learned prototypes on Mutagenicity do not show a clear recurring substructure pattern.
Across five repeated runs, we did not observe an obvious chemically meaningful concept consistently represented by either \(p_0\) or \(p_1\).
This makes the learned explanation difficult to interpret: the formulas are short, but their semantic content remains unclear because each prototype is only presented through a set of exemplar subgraphs.

Moreover, the learned rule for Class 1, \((p_0 \land p_1)\lor p_0\), is logically reducible to \(p_0\), which further suggests that the final explanation is effectively driven by a single ambiguous prototype.
In contrast, {\mymethod} extracts explicit global subgraph concepts, such as chemically meaningful functional groups or larger structured molecular subgraphs, and reuses them as literals in rules with measurable logit-level contributions.

\begin{figure}[htp]
    \centering
    \vspace{-5mm}
    \includegraphics[width=0.9\textwidth]{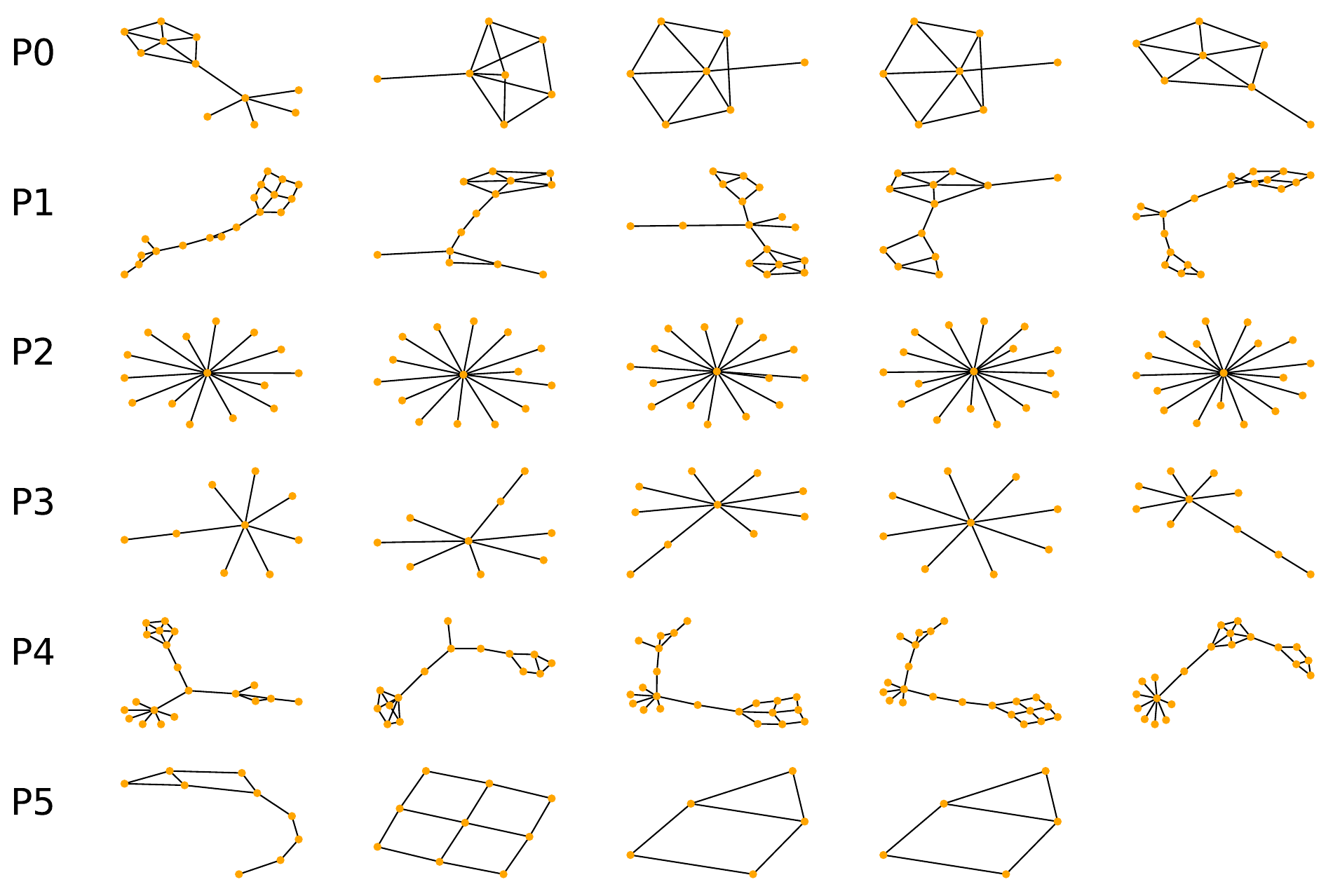}
    \vspace{1mm}
    \begin{small}
    \[
    \begin{array}{l}
    p_3 \lor p_2 \lor p_5
    \Rightarrow Class\ 0\\ \\
    p_4 \lor p_1 \lor p_0 \lor (p_2\land p_0) \lor (p_5\land p_0) \lor (p_4\land p_0) \lor (p_3\land p_0) \lor (p_2\land p_0) \lor (p_3\land p_4) \lor (p_3\land p_5) \\
    \lor (p_3\land p_2\land p_0) \lor (p_3\land p_5 \land p_0) \lor (p_3\land p_1)
    \Rightarrow Class\ 1\\
    \end{array}
    \]
    \end{small}
    \vspace{-3mm}
    \caption{
    GLGExplainer~\cite{azzolin2022global} explanation on BAMultiShapes.
    The learned prototypes \(p_i\) are illustrated by multiple exemplar subgraphs, leaving their semantic meaning to be inferred by the user rather than defined as a single grounded concept.
    The class-wise Boolean formulas below provide a global rule summary, but they do not cleanly match the known BAMultiShapes label logic: Class 0 should include plain BA, single-motif, and all-three-motif graphs, whereas Class 1 should include exactly-two-motif graphs. 
    }
    \label{fig:glg_bamult}
    \centering
    \vspace{5mm}
    \includegraphics[width=0.7\textwidth]{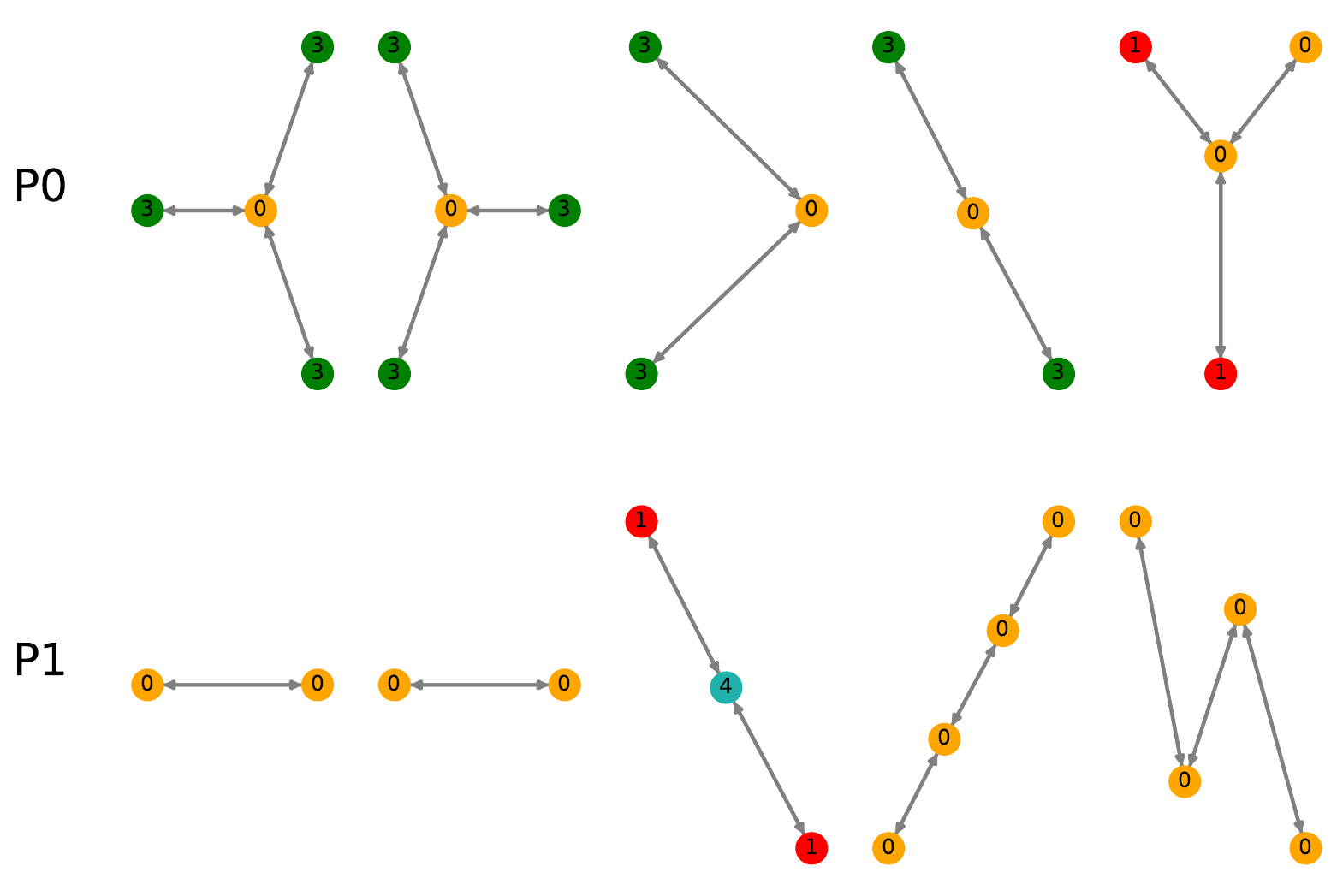}
    \vspace{1mm}
    \begin{small}
    \[
    \begin{array}{l}
    p_1
    \Rightarrow Class\ 0\\
    (p_0\land p_1)\lor p_0
    \Rightarrow Class\ 1\\
    \end{array}
    \]
    \end{small}
    \vspace{-3mm}
    \caption{
    GLGExplainer~\cite{azzolin2022global} explanation on Mutagenicity.
    Following the original setting in their paper for this dataset, we use two prototypes, \(p_0\) and \(p_1\).
    Each prototype is illustrated by multiple exemplar subgraphs rather than by a single grounded molecular concept.
    Across the displayed exemplars, neither \(p_0\) nor \(p_1\) exhibits a clear recurring chemically meaningful pattern, making their semantic meaning difficult to infer.
    }
    \label{fig:glg_mutag}
\end{figure}

\begin{figure}[htp]
    \centering
    \includegraphics[width=\textwidth]{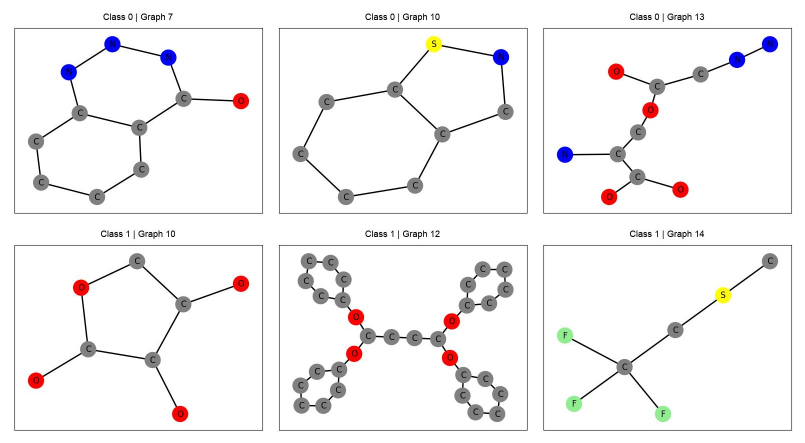}
    \caption{
    Class-wise prototype explanations generated by MAGE~\cite{yu2025mage} on Mutagenicity.
    The upper row corresponds to Class 0 and the lower row corresponds to Class 1.
    Such generation-based explanations visualize graph instances that the base GNN associates with a target class, but they do not expose how reusable concepts are composed into rules or how those rules contribute to the multiclass logits.
    }
    \label{fig:mage_mutag}
\end{figure}

\paragraph{Generation-based explanations.}
Generation-based model-level explainers such as XGNN~\cite{yuan2020xgnn}, GNNInterpreter~\cite{wang2022gnninterpreter}, and MAGE~\cite{yu2025mage} synthesize class-wise graph prototypes that the base GNN predicts as a target class.
As shown in \cref{fig:mage_mutag}, MAGE provides visual molecular instances associated with Class 0 or Class 1 on Mutagenicity.
However, these explanations do not explicitly define reusable logical rules, rule weights, test-time activations, or logit-level contributions.
They therefore answer a different question from {\mymethod}: what graph instances look class-discriminative, rather than how grounded symbolic evidence reconstructs the GNN's multiclass logits.
For this reason, we treat generation-based methods as qualitatively related but not directly comparable under our main rule-level logit reconstruction protocol. 

\begin{figure}[htp]
    \centering
    \begin{small}
    \[
    \begin{array}{lll}
    \textbf{BAMultShapes: } \\
    Class\ 0 & \neg\deg\operatorname{-greater}(16) & Score=0.999\\
    Class\ 0 & \neg\deg\operatorname{-greater}(4)\lor \deg\operatorname{-is}_n(2) & Score=0.457\\
    Class\ 0 & \deg\operatorname{-greater}(4)\land \neg\deg\operatorname{-is}_n(1) & Score=0.556\\
    Class\ 1 & \deg\operatorname{-greater}(7)\land \neg\deg\operatorname{-is}_n(3) & Score=0.416\\
    Class\ 1 & \deg\operatorname{-greater}(4)\land \neg\deg\operatorname{-is}_n(1) & Score=0.517\\
    Class\ 1 & \neg\deg\operatorname{-greater}(7)\land \neg\deg\operatorname{-is}_n(1) & Score=0.323\\ \\
    \textbf{Mutagenicity:} \\
    Class\ 0 & \neg\operatorname{2hop}(\mathrm{P})\lor \neg\operatorname{is}(\mathrm{N}) & Score=0.889\\
    Class\ 0 & \neg\operatorname{next-to}(\mathrm{N})\lor \neg\operatorname{next-to}(\mathrm{O})\lor\operatorname{is}(\mathrm{N}) & Score=0.241\\
    Class\ 0 & \neg\operatorname{next-to}(\mathrm{Cl})\land \neg\operatorname{next-to}(\mathrm{H})\land\operatorname{is}(\mathrm{C}) & Score=0.341\\
    Class\ 1 & \neg\operatorname{2hop}(\mathrm{P})\lor \neg\operatorname{is}(\mathrm{N}) & Score=0.913\\
    Class\ 1 & \neg\operatorname{next-to}(\mathrm{S})\land \neg\operatorname{next-to}(\mathrm{H})\land\operatorname{is}(\mathrm{C}) & Score=0.415\\
    Class\ 1 & \neg\operatorname{next-to}(\mathrm{H})\land(\operatorname{is}(\mathrm{C}) \lor \mathrm{NO}) & Score=0.393\\
    \end{array}
    \]
    \end{small}
    \caption{Example active concepts extracted by GCNeuron~\cite{xuanyuan2023global} on BAMultiShapes and Mutagenicity. 
    GCNeuron expresses explanations through dataset-specific concept templates, such as degree-based predicates on BAMultiShapes and atom-neighborhood predicates on Mutagenicity.
    The reported score measures the association strength between an extracted concept and a target class, but the explanation does not specify how multiple concepts are combined into an executable decision rule or how they contribute to the GNN logits.
    }
    \label{fig:gcneuron}
\end{figure}

\paragraph{Comparison with GCNeuron.}
GCNeuron explains a trained GNN by extracting class-associated active concepts from manually specified, dataset-dependent templates.
As shown in \cref{fig:gcneuron}, the templates are degree-based on BAMultiShapes and atom-neighborhood-based on Mutagenicity.
The reported score measures how strongly a concept is associated with a class, but it does not specify how multiple concepts are combined, how their activations determine the final prediction, or how they affect the multiclass logits.
Thus, GCNeuron provides concept-level summaries rather than executable rule-level explanations.
For instance, on BAMultiShapes, the same concept \(\mathrm{deg\text{-}greater}(4)\land \neg \mathrm{deg\text{-}is}_n(1)\) appears for both Class 0 and Class 1 with similar scores.
This illustrates that concept association alone does not specify a discriminative decision rule.
In contrast, {\mymethod} composes grounded concepts into logical rules, learns rule weights, instantiates rules on test graphs, and quantifies each rule's contribution to the reconstructed logits.



\end{document}